\documentclass{article}

 \usepackage[preprint, nonatbib]{neurips_2024}

\usepackage{amsmath}
\usepackage[utf8]{inputenc} 
\usepackage[T1]{fontenc}    
\usepackage{hyperref}       
\usepackage{url}            
\usepackage{booktabs}       
\usepackage{amsfonts}       
\usepackage{nicefrac}       
\usepackage{microtype}      
\usepackage{xcolor}         
\usepackage{amssymb}
\usepackage{mathtools}
\usepackage{amsthm}

\theoremstyle{plain}
\newtheorem{theorem}{Theorem}[section]

\theoremstyle{definition}
\newtheorem{definition}[theorem]{Definition}

\theoremstyle{remark}

\title{Derivation of Back-propagation for Graph Convolutional Networks using Matrix Calculus and its Application to Explainable Artificial Intelligence}

\author{%
  Yen-Che Hsiao \\
  Department of Electrical and Computer Engineering\\
  University of Connecticut\\
  Storrs CT 06269, USA \\
  \texttt{yen-che.hsiao@uconn.edu} \\
  \And
  Rongting Yue \\
  Department of Electrical and Computer Engineering\\
  University of Connecticut\\
  Storrs CT 06269, USA \\
  \AND
  Abhishek Dutta \\
  Department of Electrical and Computer Engineering\\
  University of Connecticut\\
  Storrs CT 06269, USA \\
}

\begin{document}

\maketitle

\begin{abstract}
  This paper provides a comprehensive and detailed derivation of the backpropagation algorithm for graph convolutional neural networks using matrix calculus. The derivation is extended to include arbitrary element-wise activation functions and an arbitrary number of layers. The study addresses two fundamental problems, namely node classification and link prediction. To validate our method, we compare it with reverse-mode automatic differentiation. The experimental results demonstrate that the median sum of squared errors of the updated weight matrices, when comparing our method to the approach using reverse-mode automatic differentiation, falls within the range of \(10^{-18}\) to \(10^{-14}\). These outcomes are obtained from conducting experiments on a five-layer graph convolutional network, applied to a node classification problem on Zachary's karate club social network and a link prediction problem on a drug-drug interaction network. Finally, we show how the derived closed-form solution can facilitate the development of explainable AI and sensitivity analysis.
\end{abstract}

\section{Introduction}

Graph neural network (GNN) is a neural network model for processing data represented in graph domains, encompassing cyclic, directed, and undirected graphs \cite{scarselli2008graph}. Graph convolutional network (GCN) is a type of GNN model that employs layer-wise propagation rule, operating directly on graphs \cite{kipf2016semi}. A GCN layer comprises message passing over nodes/edges, succeeded by an aggregation/pooling strategy and a fully connected layer \cite{duttadeep}.

Node classification is one of the most common research directions in graph analysis, where the objective of the task is to predict a specific class for each unlabeled node in a graph using graph information \cite{xiao2022graph}. J. Zhang \emph{et al.} \cite{zhang2023graph} applied GCN for distribution system anomaly detection, categorizing nodes into normal data, data anomalies, or event anomalies, and implemented their method on the IEEE 37-node distribution systems. Link prediction seeks to infer unobserved/missing links or predict future ones based on the connections within currently observed partial networks \cite{wang2021benchmarking}. R. Yue \emph{et al.} \cite{yue2023repurposing} applied a GCN that utilizes protein-protein interactions, drug-protein interactions, and drug-target interactions to predict potential drug molecules capable of binding with disease-related proteins.

Back-propagation (BP) was proposed by David Rumelhart, Geoffrey Hinton, and Ronald Williams as a learning procedure for training neural networks utilizing a set of input-output training samples \cite{rumelhart1986learning, THEODORIDIS2020901}. 
BP makes use of chain rule to compute the gradient of the network's error with respect to every single model parameter, allowing for the adjustment of weight and bias terms via gradient descent to reduce the error \cite{geron2022hands, werbos1988backpropagation}. Reverse mode automatic differentiation is the primary technique in the form of the back-propagation algorithm for training neural networks due to its computational efficiency, particularly for an objective functions with a large number of inputs and a scalar output \cite{baydin2018automatic}.

M. A. Nielsen \cite{nielsen2015neural} derived the BP for multi-layer perceptrons (MLPs) in a matrix-based form using the Hadamard product for numerical efficiency. N. M. Mishachev \cite{mishachev2017backpropagation} utilized Hadamard product and Kronecker product to derive an explicit matrix version of the BP equations for MLPs and reduce the number of the indices in the equations. M. Naumov \cite{naumov2017feedforward} demonstrated that the gradient of the weights can be expressed as a rank-1 and rank-t matrix for MLPs and recurrent neural networks at time step t, respectively. Y. Cheng \cite{cheng2021derivation} derived the BP algorithm for MLPs based on derivative amplification coefficients. 

Explainable Artificial Intelligence (XAI) focuses on developing AI systems that not only make accurate predictions but also provide clear, interpretable explanations for their actions and decisions, thus increasing their trustworthiness for human users \cite{hassija2024interpreting}. In XAI, sensitivity analysis is crucial because it reveals which parameters, or groups of parameters, have the most significant influence on the predictions made by machine learning models \cite{van2022comparison}.

In this work, we used matrix calculus to derive an analytic and exact closed-form solution of the gradient of the loss function with respect to weight matrices for the training of GCN. The derivation considers the problems of binary node classification and link prediction. We first considered the problem of node classification with three-layer GCN and we extended the GCN model to a multi-layer GCN with \(d\) layers and applied the procedure for a link prediction problem. To validate our approach, we applied our weight gradient calculation for node classification on a Zachary’s Karate Club graph \cite{zachary1977information} and for link prediction using a self-prepared 100-node drug-drug interaction network. Subsequently, we compared the updated weight matrix using our method to the method employing reverse mode automatic differentiation for gradient computation. Our results demonstrate a ignorable difference between our weight matrix and the weight matrix updated using the gradient from reverse mode automatic differentiation, which verifies the correctness of our method. As there is no existing work deriving the BP of GCN in matrix form, we believe it is important for the machine learning community to have this closed-form solution available. In addition, we applied the same procedure to determine the sensitivity of the loss or output with respect to the input feature matrix in GCN for XAI. The definition of notations and some properties used in this paper are provided in Appendix \ref{Notations}.

\section{Back-propagation of Graph Convolutional Network}


A basic graph structure is defined as:
\begin{equation}
G = (V,E),
\end{equation}
where \(|V|=n\) is the number of nodes in the graph and \(|E|=n_{e}\) is the number of edges. Denoting \(v_i \in V\) as a node and \(e_{ij} = (v_i, v_j) \in E\) as an edge pointing from \(v_i\) to \(v_j\), the adjacency matrix, \(\mathbf{A} \in {\{0,1\}}^{n \times n}\), is an \(n \times n\) matrix where \(a_{ij}\) equals \(1\) if the edge \(e_{ij}\) exists, and \(a_{ij}\) equals \(0\) if \(e_{ij}\) does not belong to \(E\); in addition, a graph may possess node attributes represented by the matrix \(\mathbf{H}_{0} \in \mathbb{R}^{n \times n_0}\), where \(\mathbf{h_{0_v}} \in \mathbb{R}^{n_0}\) is the feature vector of a node \(v\) with \(n_0\) features \cite{wu2020comprehensive}.


\subsection{Binary classification of nodes}

\subsubsection{Backpropagation for 3-layer GCN  with ReLU and sigmoid activation function}

We consider a 3-layer GCN defined as
\begin{equation}
\mathbf{H}_{1} = \sigma_{ReLU}(\mathbf{A}\mathbf{H}_{0}\mathbf{W}_{1}),
\label{H1def3GCNnode}
\end{equation}
\begin{equation}
\mathbf{H}_{2} = \sigma_{ReLU}(\mathbf{A}\mathbf{H}_{1}\mathbf{W}_{2}),
\label{H2def3GCNnode}
\end{equation}
\begin{equation}
\mathbf{H}_{3} = \sigma_{ReLU}(\mathbf{A}\mathbf{H}_{2}\mathbf{W}_{3}),
\label{H2def3GCNnode}
\end{equation}
\begin{equation}
\mathbf{\hat{Y}} = \sigma_{sigmoid}(\mathbf{H}_{3}),
\end{equation}
where \(\mathbf{A} \in {\{0,1\}}^{n \times n}\) is an \(n \times n\) adjacency matrix, \(\mathbf{H}_{0} \in \mathbb{R}^{n \times n_{0}}\) is the feature matrix  for \(n\) nodes with \(n_{0}\) features, \(\mathbf{\hat{Y}} \in \mathbb{R}^{n \times n_{3}}\) is the output matrix, \(n_{3}=1\) for binary node classification, and \(\mathbf{W}_{1} \in \mathbb{R}^{n_{0} \times n_{1}}\), \(\mathbf{W}_{2} \in \mathbb{R}^{n_{1} \times n_{2}}\), and \(\mathbf{W}_{3} \in \mathbb{R}^{n_{2} \times n_{3}}\) are trainable parameter matrices, \(n_{1}\), \(n_{2}\), and \(n_{3}\) are the number of features in \(\mathbf{H}_{1}\in \mathbb{R}^{n \times n_{1}}\), \(\mathbf{H}_{2}\in \mathbb{R}^{n \times n_{2}}\), and \(\mathbf{H}_{3}\in \mathbb{R}^{n \times n_{3}}\), respectively, \(\sigma_{ReLU}\) denotes an element-wise rectified linear unit (ReLU) function \(\sigma_{ReLU}(v)=max(v,0)\) \cite{nair2010rectified}, \(\sigma_{sigmoid}\) denotes an element-wise sigmoid function \(\sigma_{sigmoid}(v)=\frac{1}{1+e^{-v}}\)..

The loss function of the 3-layer GCN for binary node classification can be defined as
\begin{equation}
L = L_1 + L_2,
\label{nodeloss-2}
\end{equation}
where 
\begin{equation}
L_1 = -\sum_{i=1}^{n} \sum_{j=1}^{n_{3}} (y_{ij}ln(\hat{y}_{ij})),
\label{nodeloss-2-L1}
\end{equation}
\begin{equation}
L_2 = -\sum_{i=1}^{n} \sum_{j=1}^{n_{3}} ((1-y_{ij})ln(1-\hat{y}_{ij})),
\label{nodeloss-2-L2}
\end{equation}
where \(y_{ij}\) denotes the element in the \(i\)th row and \(j\)th column of the ground truth of nodes \(\mathbf{Y} \in \mathbb{R}^{n \times n_{3}}\), \(\hat{y}_{ij}\) denotes the element in the \(i\)th row and \(j\)th column of \(\mathbf{\hat{Y}}\).

To derive the derivative of the loss function \(L\) with respect to the third weight matrix \(\mathbf{W}_{3}\), we first write 
\begin{equation}
\frac{\partial L}{\partial \mathbf{W}_{3}}=\frac{\partial L_1}{\partial \mathbf{W}_{3}}+\frac{\partial L_2}{\partial \mathbf{W}_{3}}
\label{Dnodeloss}.
\end{equation}
Then, the derivative of \(L_1\) in (\ref{nodeloss-2-L1}) with respect to \(\mathbf{W}_{3}\) can be written as
\begin{equation}
\begin{aligned}
\frac{\partial L_1}{\partial \mathbf{W}_{3}} &= -\sum_{i=1}^{n} \sum_{j=1}^{n_{3}} y_{ij}\cdot \frac{\partial ln(\hat{y}_{ij})}{\partial \mathbf{W}_{3}}\\ 
&= -\sum_{i=1}^{n} \sum_{j=1}^{n_{3}} y_{ij}\cdot \frac{\partial ln(\sigma_{sigmoid}(h_{3ij}))}{\partial \mathbf{W}_{3}}\\ 
&= -\sum_{i=1}^{n} \sum_{j=1}^{n_{3}} y_{ij}\cdot \frac{1}{\hat{y}_{ij}} \cdot \hat{y}_{ij} \cdot(1-\hat{y}_{ij})\cdot \frac{\partial h_{3ij}}{\partial \mathbf{W}_{3}}\\ 
&= -\sum_{i=1}^{n} \sum_{j=1}^{n_{3}} y_{ij}\cdot(1-\hat{y}_{ij})\cdot \frac{\partial h_{3ij}}{\partial \mathbf{W}_{3}} .
\label{2-1-1-E1}
\end{aligned}
\end{equation}
Similarly, the derivative of \(L_2\) in (\ref{nodeloss-2-L2}) with respect to \(\mathbf{W}_{3}\) can be written as
\begin{equation}
\begin{aligned}
\frac{\partial L_2}{\partial \mathbf{W}_{3}} &= -\sum_{i=1}^{n} \sum_{j=1}^{n_{3}} (1-y_{ij})\cdot \frac{\partial ln(1-\hat{y}_{ij})}{\partial \mathbf{W}_{3}}\\ 
&= \sum_{i=1}^{n} \sum_{j=1}^{n_{3}} (1-y_{ij})\cdot\hat{y}_{ij}\cdot \frac{\partial h_{3ij}}{\partial \mathbf{W}_{3}} .
\label{2-1-1-E2}
\end{aligned}
\end{equation}

The last term, \(\frac{\partial h_{3ij}}{\partial \mathbf{W}_{3}}\), in (\ref{2-1-1-E1}) and (\ref{2-1-1-E2}) can be written as \(\frac{\partial h_{3ij}}{\partial \mathbf{W}_{3}}=\frac{\partial \sigma_{ReLU}(\mathbf{A}_{i*}\mathbf{H}_{2}\mathbf{W}_{3_{*j}})}{\partial \mathbf{W}_{3}}\) using (\ref{H2def3GCNnode}), where \(\mathbf{A}_{i*}\) is the \(i\)th row of \(\mathbf{A}\) and \(\mathbf{W}_{3_{*j}}\) is the \(j\)th column of \(\mathbf{W}_{3}\). The expression, \(\frac{\partial \sigma_{ReLU}(\mathbf{A}_{i*}\mathbf{H}_{2}\mathbf{W}_{3_{*j}})}{\partial \mathbf{W}_{3}}\), is solved using the following definitions and theorem.

\begin{definition}
\label{def:elemact}
An element-wise activation function \(\mathbf{\Sigma}: \mathbb{R}^{p \times q} \rightarrow \mathbb{R}^{p \times q}\) is defined as a multivariate matrix-valued function \cite{ostwald2021induction}:
\begin{equation}
\mathbf{\Sigma}: \mathbf{X} \mapsto \mathbf{\Sigma}(\mathbf{X}) \coloneqq \begin{bmatrix}
\sigma(x_{11}) & \cdots & \sigma(x_{1q})\\
\vdots & & \vdots \\
\sigma(x_{p1}) & \cdots & \sigma(x_{pq})
\end{bmatrix},
\end{equation}
where
\(\mathbf{X} \in \mathbb{R}^{p \times q}\) and \(\sigma: \mathbb{R} \rightarrow \mathbb{R}, x_{ij} \mapsto \sigma(x_{ij})\) for all \(i=1,2,...,p\) and \(j=1,2,...,q\).
\end{definition}

\begin{definition}
\label{def:derivelemact}
The derivative of an element-wise activation function \(\mathbf{\Sigma}(\mathbf{X})\) \((p\times q)\) with respect to a matrix \(\mathbf{X} \in \mathbb{R}^{p \times q}\) is written as:
\begin{equation}
\begin{aligned}
\mathbf{\Sigma}'(\mathbf{X}) &= 
\frac{\partial \mathbf{\Sigma}(\mathbf{X})}{\partial \mathbf{X}} =\begin{bmatrix}
\sigma'(x_{11}) & \cdots & \sigma'(x_{1q})\\
\vdots & & \vdots \\
\sigma'(x_{p1}) & \cdots & \sigma'(x_{pq})
\end{bmatrix}\\
\end{aligned},
\end{equation}
where \(\sigma'(v)=\lim_{u\to0} \frac{\sigma(v+u)-\sigma(v)}{u}\) is defined as the derivative of a real-valued function \(\sigma\).
\end{definition}

\begin{theorem}
\label{thm:arbact}
Let \(\mathbf{F}: \mathbb{R}^{p \times q} \rightarrow \mathbb{R}^{m \times n}\) be a \(m\times n\) multivariate matrix-valued function of a \(p\times q\) matrix \(\mathbf{W}\in\mathbb{R}^{p \times q}\), the derivative of \(\mathbf{\Sigma}(\mathbf{F}(\mathbf{W}))\) with respect to \(\mathbf{W}\) can be written as:
\begin{equation}
\frac{\partial \mathbf{\Sigma}(\mathbf{F}(\mathbf{W}))}{\partial \mathbf{W}} = 
(\mathbf{J}_{p \times q} \otimes \mathbf{\Sigma}'(\mathbf{F}(\mathbf{W}))) \odot \frac{\partial \mathbf{F}(\mathbf{W})}{\partial \mathbf{W}},
\label{ArbitraryActivationMatrixDerivativeFinal}
\end{equation}
where \(\otimes\) is the Kronecker product in (\ref{Kronecker}), \(\odot\) is the Hadamard product in (\ref{Hadamard}), and \(J_{p \times q}\in\mathbb{R}^{p \times q}\) is an all 1 matrix.
\end{theorem}
\begin{proof} The proof can be found in Appendix \ref{ProofTheorem}.
\end{proof}

From theorem \ref{thm:arbact}, we can write the last term in (\ref{2-1-1-E1}) and (\ref{2-1-1-E2}) as
\begin{equation}
\begin{aligned}
\frac{\partial h_{3ij}}{\partial \mathbf{W}_{3}} &= \frac{\partial \sigma_{ReLU}(\mathbf{A}_{i*}\mathbf{H}_{2}\mathbf{W}_{3_{*j}})}{\partial \mathbf{W}_{3}}\\ 
&= (\mathbf{J}_{n_2 \times n_3} \otimes \sigma_{ReLU}'(\mathbf{A}_{i*}\mathbf{H}_{2}\mathbf{W}_{3_{*j}})) \odot \frac{\partial \mathbf{A}_{i*}\mathbf{H}_{2}\mathbf{W}_{3_{*j}}}{\partial \mathbf{W}_{3}} .
\label{2-1-1-E3}
\end{aligned}
\end{equation}

Using (\ref{T2}), the last term in (\ref{2-1-1-E3}) can be written as 
\begin{equation}
\begin{aligned}
\frac{\partial \mathbf{A}_{i*}\mathbf{H}_{2}\mathbf{W}_{3_{*j}}}{\partial \mathbf{W}_{3}} &= \frac{\partial \mathbf{A}_{i*}\mathbf{H}_{2}}{\partial \mathbf{W}_{3}}(\mathbf{I}_{n_3} \otimes\mathbf{W}_{3_{*j}})+(\mathbf{I}_{n_2} \otimes(\mathbf{A}_{i*}\mathbf{H}_{2}))\frac{\partial \mathbf{W}_{3_{*j}}}{\partial \mathbf{W}_{3}}\\ 
&= (\mathbf{I}_{n_2} \otimes(\mathbf{A}_{i*}\mathbf{H}_{2}))\frac{\partial \mathbf{W}_{3_{*j}}}{\partial \mathbf{W}_{3}} .
\label{2-1-1-E4}
\end{aligned}
\end{equation}

Using (\ref{T5}), (\ref{T2}), and (\ref{T3}), the last term in (\ref{2-1-1-E4}) can be written as
\begin{equation}
\begin{aligned}
\frac{\partial \mathbf{W}_{3_{*j}}}{\partial \mathbf{W}_{3}} &=\frac{\partial \mathbf{W}_{3}\underset{(n_3)}{\mathbf{e}_{j}}}{\partial \mathbf{W}_{3}}\\ 
&=\frac{\partial \mathbf{W}_{3}}{\partial \mathbf{W}_{3}}(\mathbf{I}_{n_3}\otimes \underset{(n_3)}{\mathbf{e}_{j}})
+(\mathbf{I}_{n_2}\otimes \mathbf{W}_{3})\frac{\partial \underset{(n_3)}{\mathbf{e}_{j}}}{\partial \mathbf{W}_{3}}\\ 
&=\mathbf{\bar{U}}_{n_2\times n_3}(\mathbf{I}_{n_3}\otimes \underset{(n_3)}{\mathbf{e}_{j}}) ,
\label{2-1-1-E5}
\end{aligned}
\end{equation}
where \(\underset{(n_3)}{\mathbf{e}_{j}}\) is a \(n_3\)-dimensional column vector which has “1” in the \(j\)th row and zero elsewhere and \(\mathbf{\bar{U}}_{n_2\times n_3}\in\mathbb{R}^{n_2^2\times n_3^2}\) is a permutation related matrix defined in (\ref{RelateMatrix}).

Thus, the derivative of the loss function \(L\) with respect to the third weight matrix \(\mathbf{W}_{3}\) in (\ref{Dnodeloss}) can be written as
\begin{align}
\frac{\partial L}{\partial \mathbf{W}_{3}} &= -\sum_{i=1}^{n} \sum_{j=1}^{n_{3}} ((y_{ij}-\hat{y}_{ij})\cdot \frac{\partial h_{3ij}}{\partial \mathbf{W}_{3}}) \label{2-1-1-E6-3}\\
&= -\sum_{i=1}^{n} \sum_{j=1}^{n_{3}} ((y_{ij}-\hat{y}_{ij}) \nonumber \cdot((\mathbf{J}_{n_2 \times n_3}\otimes \sigma_{ReLU}'(\mathbf{A}_{i*}\mathbf{H}_{2}\mathbf{W}_{3_{*j}})) \nonumber \\
&\quad\odot ((\mathbf{I}_{n_2} \otimes(\mathbf{A}_{i*}\mathbf{H}_{2}))\mathbf{\bar{U}}_{n_2\times n_3}(\mathbf{I}_{n_3}\otimes \underset{(n_3)}{\mathbf{e}_{j}})))) .\label{2-1-1-E6-6} 
\end{align}

Using (\ref{2-1-1-E6-3}), the derivative of the loss function \(L\) in (\ref{nodeloss-2}) with respect to the second weight matrix \(\mathbf{W}_{2}\) can be written as
\begin{align}
\frac{\partial L}{\partial \mathbf{W}_{2}} 
&= -\sum_{i=1}^{n} \sum_{j=1}^{n_{3}} ((y_{ij}-\hat{y}_{ij})\cdot \frac{\partial h_{3ij}}{\partial \mathbf{W}_{2}}) \label{2-1-1-E7-1} \\ 
&= -\sum_{i=1}^{n} \sum_{j=1}^{n_{3}} ((y_{ij}-\hat{y}_{ij})\cdot ((\mathbf{J}_{n_1 \times n_2}\otimes \sigma_{ReLU}'(\mathbf{A}_{i*}\mathbf{H}_{2}\mathbf{W}_{3_{*j}})) \odot \frac{\partial \mathbf{A}_{i*}\mathbf{H}_{2}\mathbf{W}_{3_{*j}}}{\partial \mathbf{W}_{2}})) ,\label{2-1-1-E7-2} 
\end{align}
where (\ref{2-1-1-E7-2}) follows from theorem \ref{thm:arbact}.

Using (\ref{T2}), the last term of (\ref{2-1-1-E7-2}) can be written as 
\begin{align}
\frac{\partial \mathbf{A}_{i*}\mathbf{H}_{2}\mathbf{W}_{3_{*j}}}{\partial \mathbf{W}_{2}} 
&= (\mathbf{I}_{n_1}\otimes \mathbf{A}_{i*})\frac{\partial \mathbf{H}_{2}}{\partial \mathbf{W}_{2}}(\mathbf{I}_{n_2}\otimes \mathbf{W}_{3_{*j}}) . \label{2-1-1-E8-5} 
\end{align}

The derivative of \(\mathbf{H}_{2}\) with respect to \(\mathbf{W}_{2}\) in (\ref{2-1-1-E8-5}) can be written as
\begin{align}
\frac{\partial \mathbf{H}_{2}}{\partial \mathbf{W}_{2}} 
&= \frac{\partial \sigma_{ReLU}(\mathbf{A}\mathbf{H}_{1}\mathbf{W}_{2})}{\partial \mathbf{W}_{2}} \label{2-1-1-E9-1} \\ 
&= (\mathbf{J}_{n_1 \times n_2}\otimes \sigma_{ReLU}'(\mathbf{A}\mathbf{H}_{1}\mathbf{W}_{2})) \odot \frac{\partial \mathbf{A}\mathbf{H}_{1}\mathbf{W}_{2}}{\partial \mathbf{W}_{2}}) \label{2-1-1-E9-2} \\ 
&= (\mathbf{J}_{n_1 \times n_2}\otimes \sigma_{ReLU}'(\mathbf{A}\mathbf{H}_{1}\mathbf{W}_{2})) \odot ((\mathbf{I}_{n_1}\otimes (\mathbf{A}\mathbf{H}_{1}))\frac{\partial \mathbf{W}_{2}}{\partial \mathbf{W}_{2}}) \label{2-1-1-E9-3} \\ 
&= (\mathbf{J}_{n_1 \times n_2}\otimes \sigma_{ReLU}'(\mathbf{A}\mathbf{H}_{1}\mathbf{W}_{2})) \odot ((\mathbf{I}_{n_1}\otimes (\mathbf{A}\mathbf{H}_{1}))\mathbf{\bar{U}}_{n_1\times n_2}), \label{2-1-1-E9-4} 
\end{align}
where (\ref{2-1-1-E9-1}) follows from (\ref{H2def3GCNnode}), (\ref{2-1-1-E9-2}) follows from (\ref{2-1-1-E3}), (\ref{2-1-1-E9-3}) follows from (\ref{T2}), and (\ref{2-1-1-E9-4}) follows from (\ref{T3}).

The derivative of the loss in (\ref{nodeloss-2}) with respect to the first weight matrix \(\mathbf{W}_{1}\) is
\begin{align}
\frac{\partial L}{\partial \mathbf{W}_{1}} &= -\sum_{i=1}^{n} \sum_{j=1}^{n_{3}} ((y_{ij}-\hat{y}_{ij}) \cdot((\mathbf{J}_{n_0 \times n_1}\otimes \sigma_{ReLU}'(\mathbf{A}_{i*}\mathbf{H}_{2}\mathbf{W}_{3_{*j}})) \nonumber \\
&\qquad\qquad\odot ((\mathbf{I}_{n_0} \otimes\mathbf{A}_{i*})\frac{\partial \mathbf{H}_{2}}{\partial \mathbf{W}_{1}}(\mathbf{I}_{n_1}\otimes \mathbf{W}_{3_{*j}})))), \label{2-1-2-E4}
\end{align}
where (\ref{2-1-2-E4}) follows from (\ref{2-1-1-E7-2}) and 
\begin{align}
\frac{\partial \mathbf{H}_{2}}{\partial \mathbf{W}_{1}}
&=\frac{\partial \sigma_{ReLU}(\mathbf{A}\mathbf{H}_{1}\mathbf{W}_{2})}{\partial \mathbf{W}_{1}} \label{2-1-2-E5-1} \\
&=(\mathbf{J}_{n_0 \times n_1}\otimes \sigma_{ReLU}'(\mathbf{A}\mathbf{H}_{1}\mathbf{W}_{2})) \odot ((\mathbf{I}_{n_0} \otimes\mathbf{A})\frac{\partial \mathbf{H}_{1}}{\partial \mathbf{W}_{1}}(\mathbf{I}_{n_1}\otimes \mathbf{W}_{2})) \label{2-1-2-E5-2},
\end{align}
where (\ref{2-1-2-E5-1}) follows from (\ref{H2def3GCNnode}), (\ref{2-1-2-E5-2}) follows from theorem (\ref{thm:arbact}) and (\ref{T2}), and
\begin{align}
\frac{\partial \mathbf{H}_{1}}{\partial \mathbf{W}_{1}}
&=(\mathbf{J}_{n_0 \times n_1}\otimes \sigma_{ReLU}'(\mathbf{A}\mathbf{H}_{0}\mathbf{W}_{1})) \odot ((\mathbf{I}_{n_0} \otimes(\mathbf{A}\mathbf{H}_{0}))\mathbf{\bar{U}}_{n_0\times n_1})\label{2-1-2-E6-1},
\end{align}
where (\ref{2-1-2-E6-1}) follows from (\ref{2-1-1-E9-4}).

\subsubsection{Back-propagation for multi-layer GCN  with ReLU and sigmoid activation function}

We consider a \(d\)-layer GCN defined as
\begin{equation}
\mathbf{\hat{Y}} = \sigma_{sigmoid}(\mathbf{H}_{d}),
\end{equation}
where 
\begin{equation}
\mathbf{H}_{d} = \sigma_{ReLU}(\mathbf{A}\mathbf{H}_{d-1}\mathbf{W}_{d}),
\end{equation}
\(d\in \mathbb{Z}^+\), \(\mathbf{A} \in {\{0,1\}}^{n \times n}\) is an \(n \times n\) adjacency matrix, \(\mathbf{H}_{d-1} \in \mathbb{R}^{n \times n_{d-1}}\) is the feature matrix  for \(n\) nodes with \(n_{d-1}\) features, \(\mathbf{\hat{Y}} \in \mathbb{R}^{n \times n_{d}}\) is the output matrix, \(n_{d}=1\) for binary node classification, and \(\mathbf{W}_{1} \in \mathbb{R}^{n_{0} \times n_{1}}\), \(\mathbf{W}_{2} \in \mathbb{R}^{n_{1} \times n_{2}}\), \(\cdots\), and \(\mathbf{W}_{d} \in \mathbb{R}^{n_{d-1} \times n_{d}}\) are trainable parameter matrices.

By observing (\ref{2-1-1-E6-6}), (\ref{2-1-1-E7-2}), (\ref{2-1-1-E8-5}), (\ref{2-1-1-E9-4}), (\ref{2-1-2-E4}), (\ref{2-1-2-E5-2}), and (\ref{2-1-2-E6-1}), if the loss function of the \(d\)-layer GCN for node classification is defined as
\begin{equation}
L = -\sum_{i=1}^{n} \sum_{j=1}^{n_{d}} (y_{ij}ln(\hat{y}_{ij})+(1-y_{ij})ln(1-\hat{y}_{ij})),
\label{nodeloss-3}
\end{equation}
the derivative of the loss in (\ref{nodeloss-3}) with respect to the \(s\)th weight matrix \(\mathbf{W}_{s}\) is
\begin{align}
\frac{\partial L}{\partial \mathbf{W}_{s}} 
&=-\sum_{i=1}^{n} \sum_{j=1}^{n_{d}} ((y_{ij}-\hat{y}_{ij}) \cdot((\mathbf{J}_{(s-1) \times s}\otimes \sigma_{ReLU}'(\mathbf{A}_{i*}\mathbf{H}_{d-1}\mathbf{W}_{d_{*j}})) \nonumber \\
&\qquad\qquad\odot \frac{\partial \mathbf{A}_{i*}\mathbf{H}_{d-1}\mathbf{W}_{d_{*j}}}{\partial \mathbf{W}_{s}})) ,\label{2-1-3-E1-1} 
\end{align}
where \(s\in \mathbb{Z}^+\), \(s\leq d\), 
\begin{equation}
\begin{aligned}
\frac{\partial \mathbf{A}_{i*}\mathbf{H}_{d-1}\mathbf{W}_{d_{*j}}}{\partial \mathbf{W}_{s}} =\begin{cases}
&\begin{aligned}
&(\mathbf{I}_{n_{s-1}} \otimes(\mathbf{A}_{i*}\mathbf{H}_{d-1}))\cdot\mathbf{\bar{U}}_{n_{s-1}\times n_s}(\mathbf{I}_{n_s}\otimes \underset{(n_d)}{\mathbf{e}_{j}}),
\end{aligned}\quad \text{if $s=d$}\\
&\begin{aligned}
&(\mathbf{I}_{n_{s-1}} \otimes\mathbf{A}_{i*})\cdot\frac{\partial \mathbf{H}_{d-1}}{\partial \mathbf{W}_{s}}(\mathbf{I}_{n_s}\otimes \mathbf{W}_{d_{*j}}),
\end{aligned}\quad \text{if $s<d$}
\end{cases}
\end{aligned}
\label{2-1-3-E2}
\end{equation}
and
\begin{equation}
\begin{aligned}
&\frac{\partial \mathbf{H}_{d-1}}{\partial \mathbf{W}_{s}} \\
&=\begin{cases}
&\begin{aligned}
&(\mathbf{J}_{n_{s-1} \times n_s}\otimes \sigma_{ReLU}'(\mathbf{A}\mathbf{H}_{d-2}\mathbf{W}_{d-1})) \odot ((\mathbf{I}_{n_{s-1}} \otimes(\mathbf{A}\mathbf{H}_{d-2}))\mathbf{\bar{U}}_{n_{s-1}\times n_s}),
\end{aligned}\quad\\
&\text{if $s=d-1$}\\
&\begin{aligned}
&(\mathbf{J}_{n_{s-1} \times n_s}\otimes \sigma_{ReLU}'(\mathbf{A}\mathbf{H}_{d-2}\mathbf{W}_{d-1})) \odot ((\mathbf{I}_{n_{s-1}} \otimes\mathbf{A})\frac{\partial \mathbf{H}_{d-2}}{\partial \mathbf{W}_{s}}(\mathbf{I}_{n_s} \otimes\mathbf{W}_{d-1})),
\end{aligned}\quad \\
&\text{if $s< d-1$}.
\end{cases}
\end{aligned}
\label{2-1-3-E3}
\end{equation}

\subsubsection{Back-propagation for multi-layer GCN  with arbitrary activation functions}
\label{BPMLGCNAAF_NC}

We consider a \(d\)-layer GCN defined as
\begin{equation}
\mathbf{\hat{Y}} = \mathbf{\Sigma}_{d+1}(\mathbf{H}_{d}),
\label{dGCNnode}
\end{equation}
where 
\begin{equation}
\mathbf{H}_{d} = \mathbf{\Sigma}_{d}(\mathbf{A}\mathbf{H}_{d-1}\mathbf{W}_{d}),
\end{equation}
\(d\in \mathbb{Z}^+\), \(\mathbf{A} \in {\{0,1\}}^{n \times n}\) is an \(n \times n\) adjacency matrix, \(\mathbf{H}_{d-1} \in \mathbb{R}^{n \times n_{d-1}}\) is the feature matrix  for \(n\) nodes with \(n_{d-1}\) features, \(\mathbf{\hat{Y}} \in \mathbb{R}^{n \times n_{d}}\) is the output matrix, \(n_{d}=1\) for binary node classification, \(\mathbf{W}_{1} \in \mathbb{R}^{n_{0} \times n_{1}}\), \(\mathbf{W}_{2} \in \mathbb{R}^{n_{1} \times n_{2}}\), \(\cdots\), and \(\mathbf{W}_{d} \in \mathbb{R}^{n_{d-1} \times n_{d}}\) are trainable parameter matrices, and \(\mathbf{\Sigma}_{1}\), \(\mathbf{\Sigma}_{2}\), \(\cdots\), and \(\mathbf{\Sigma}_{d+1}\) are any element-wise activation function.

If the loss function of the \(d\)-layer GCN for node classification is defined as
\begin{equation}
L = -\sum_{i=1}^{n} \sum_{j=1}^{n_{d}} (y_{ij}ln(\hat{y}_{ij})+(1-y_{ij})ln(1-\hat{y}_{ij})),
\label{nodeloss-4}
\end{equation}
the derivative of the loss in (\ref{nodeloss-4}) with respect to the \(s\)th weight matrix \(\mathbf{W}_{s}\) is
\begin{align}
\frac{\partial L}{\partial \mathbf{W}_{s}} 
&=-\sum_{i=1}^{n} \sum_{j=1}^{n_{d}} (y_{ij}\frac{\partial ln(\hat{y}_{ij})}{\partial \mathbf{W}_{s}}+(1-\hat{y}_{ij}) \frac{\partial ln(1-\hat{y}_{ij})}{\partial \mathbf{W}_{s}}) \nonumber \\
&=-\sum_{i=1}^{n} \sum_{j=1}^{n_{d}} ((\frac{y_{ij}}{\hat{y}_{ij}}-\frac{1-y_{ij}}{1-\hat{y}_{ij}})\mathbf{\Sigma}_{d+1}'(h_{d_{ij}}) \cdot\frac{\partial \mathbf{\Sigma}_{d}(\mathbf{A}_{i*}\mathbf{H}_{d-1}\mathbf{W}_{d_{*j}})}{\partial \mathbf{W}_{s}}) \label{2-1-4-E1-2} \\
&=-\sum_{i=1}^{n} \sum_{j=1}^{n_{d}} (\frac{y_{ij}-\hat{y}_{ij}}{\hat{y}_{ij}(1-\hat{y}_{ij})}\mathbf{\Sigma}_{d+1}'(h_{d_{ij}}) \cdot((\mathbf{J}_{(s-1) \times s}\otimes \mathbf{\Sigma}_{d}'(\mathbf{A}_{i*}\mathbf{H}_{d-1}\mathbf{W}_{d_{*j}})) \nonumber \\
&\qquad\qquad\odot \frac{\partial \mathbf{A}_{i*}\mathbf{H}_{d-1}\mathbf{W}_{d_{*j}}}{\partial \mathbf{W}_{s}})) , \label{2-1-4-E1-3} 
\end{align}
where \(s\in \mathbb{Z}^+\), \(s\leq d\), (\ref{2-1-4-E1-2}) follows from chain rule and definition \ref{def:derivelemact}, (\ref{2-1-4-E1-3}) follows from theorem (\ref{thm:arbact}), \(\frac{\partial \mathbf{A}_{i*}\mathbf{H}_{d-1}\mathbf{W}_{d_{*j}}}{\partial \mathbf{W}_{s}}\) is the same as in (\ref{2-1-3-E2})
, and
\begin{equation}
\begin{aligned}
&\frac{\partial \mathbf{H}_{d-1}}{\partial \mathbf{W}_{s}} \\
&=\begin{cases}
&\begin{aligned}
&(\mathbf{J}_{n_{s-1} \times n_s}\otimes \mathbf{\Sigma}_{d-1}'(\mathbf{A}\mathbf{H}_{d-2}\mathbf{W}_{d-1})) \odot ((\mathbf{I}_{n_{s-1}} \otimes(\mathbf{A}\mathbf{H}_{d-2}))\mathbf{\bar{U}}_{n_{s-1}\times n_s})
\end{aligned}\quad \\
&\text{if $s=d-1$}\\
&\begin{aligned}
&(\mathbf{J}_{n_{s-1} \times n_s}\otimes \mathbf{\Sigma}_{d-1}'(\mathbf{A}\mathbf{H}_{d-2}\mathbf{W}_{d-1})) \odot ((\mathbf{I}_{n_{s-1}} \otimes\mathbf{A})\frac{\partial \mathbf{H}_{d-2}}{\partial \mathbf{W}_{s}}(\mathbf{I}_{n_s} \otimes\mathbf{W}_{d-1}))
\end{aligned}\quad \\
&\text{if $s< d-1$}.
\end{cases}
\end{aligned}
\label{2-1-4-E3}
\end{equation} 

Following a similar procedure, we can derive the sensitivity of the loss with respect to the feature matrix \(\mathbf{H}_{0}\), as shown in Appendix \ref{sens_node}.

\subsection{Link prediction}

For the derivation of back-propagation and sensitivity in a multi-layer GCN with arbitrary activation functions for the link prediction problem, please refer to Appendices \ref{DerLink} and \ref{sens_link}, respectively.

\section{Experiments}

In this section, we demonstrate the correctness of our method through several experiments considering a node classification problem on Zachary’s Karate Club \cite{zachary1977information} in Figure~\ref{a-Karate_Drug} (\textbf{a}) and a link prediction problem on a drug-drug interaction (DDI) network in Figure~\ref{a-Karate_Drug} (\textbf{b}). We adopt the stochastic gradient descent (SGD) for model training. We compare our method to the gradient computation using reverse mode automatic differentiation in Pytorch \cite{paszke2019pytorch} by computing the sum of squared error (SSE) defined as
\begin{equation}
SSE = \sum_{i=1}^{n_{s-1}} \sum_{j=1}^{n_{s}} (w_{s_{ij}}^{(AD)}-w_{s_{ij}}^{(KP)})^2,
\label{SSE}
\end{equation}
where \(w_{s_{ij}}^{(AD)}\) is the \(i\)th row and the \(j\)th column of the \(s\)th weight matrix \(\mathbf{W}_{s}\) updated using SGD with reverse mode automatic differentiation and \(w_{s_{ij}}^{(KP)}\) is the \(i\)th row and the \(j\)th column of the \(s\)th weight matrix \(\mathbf{W}_{s}\) updated using SGD with our matrix-based method.

In addition, we determine the sensitivity of the loss with respect to the feature matrix \(\mathbf{H}_{0}\) for node classification and link prediction to demonstrate the application of our method to XAI.

The python codes are available at: \url{https://github.com/AnnonymousForPapers/GCN-proof/tree/main}.


\subsection{Node classification}

We tested our method for node classification on Zachary’s Karate Club \cite{zachary1977information} (see Appendix \ref{DataNode} for more details). 


\subsubsection{1-layer GCN with identity function and sigmoid activation function}
\label{Node1GCN}

The \(1\)-layer GCN is defined by (\ref{dGCNnode}) and the loss function is defined by (\ref{nodeloss-4}), where \(d=1\) is the number of layers, \(\mathbf{H}_{0} = \mathbf{I}_{n} \in \mathbb{R}^{n \times n}\) is the feature matrix for \(n=34\) nodes with \(n=34\) features, the matrix \(\mathbf{\hat{Y}} \in \mathbb{R}^{n \times n_{d}}\) represents the ground truth of nodes, indicating class assignments per node, \(\mathbf{\hat{Y}} \in \mathbb{R}^{n \times n_{d}}\) denotes the final layer predictions of the GCN model, \(n_d=1\) for binary node classification, \(\mathbf{W}_{1} \in \mathbb{R}^{n_{0} \times n_{d}}\) is a trainable parameter matrix, \(n_0=34\), \(\mathbf{\Sigma}_{1}\) is an element-wise identity function, and \(\mathbf{\Sigma}_{2}\) is an element-wise sigmoid function.

The GCN is trained using SGD with a learning rate of 0.1 and with 100 iterations for both methods. In Figure~\ref{node_evolution}, the evolution of the karate club network shows that the two graph updated by the two methods have the same loss, accuracy, and classification results. In Figure~\ref{dW_small} (\textbf{a}), it shows the SSE of the weight matrix \(\mathbf{W}_{1}\) between the reverse mode automatic differentiation and our matrix-based method is lower than \(10^{-13}\). This indicates that our analytic expression aligns with the exact method, and the small difference might be due to the precision of different datatypes.

Next, we show the loss sensitivity with respect to the input feature matrix, computed using the result in Appendix \ref{sens_node}. In Figure \ref{Sens} (\textbf{a}), it is shown that as training progresses, the sensitivity decreases. This result is expected since the input feature matrix of the GCN for the Karate graph is an identity matrix. The prediction of the class of a node depends only on its edges rather than the input features. Therefore, changes in the input matrix should not significantly affect the prediction from the trained GCN, as it should learn to make predictions based solely on the existence of edges between nodes.

Additional experimental results on Zachary’s Karate Club \cite{zachary1977information} are provided in Appendix \ref{Node5GCN}.

\subsection{Link prediction}

We tested our method for link prediction on a 10-node DDI network (see Appendix \ref{DataLink} for more details).

\subsubsection{2-layer GCN}
\label{Link1GCN}

The \(2\)-layer GCN is defined by (\ref{def2GCNlink}) and the loss function is defined by (\ref{Linkloss}), where \(d=2\) is the number of layers, the matrix \(\mathbf{\hat{Y}} \in \mathbb{R}^{n \times n}\) represents ground-truth adjacency matrix, \(\mathbf{\hat{Y}} \in \mathbb{R}^{n \times n}\) denotes the final link predictions of the GCN model, \(\mathbf{W}_{1} \in \mathbb{R}^{n_{0} \times n_{1}}\) and \(\mathbf{W}_{2} \in \mathbb{R}^{n_{1} \times n_{2}}\) are trainable parameter matrices, \(n_0=20\), \(n_1=10\), \(n_2=5\), \(\mathbf{\Sigma}_{1}\) is an element-wise ReLU function \cite{nair2010rectified}, \(\mathbf{\Sigma}_{2}\) is an element-wise identity function, and \(\mathbf{\Sigma}_{3}\) is an element-wise sigmoid function.

The GCN is trained using SGD with a learning rate of 0.01 and with 150 iterations for both methods. In each iteration, thirteen negative edges, which is the same number as the number of positive edges, are uniformly sampled from a set of all the unconnected edges for the GCN training using the reverse mode automatic differentiation. The sampled negative edges in each iteration are saved in a list and used in the training of the GCN using our matrix-based method. In Figure~\ref{link_evolution}, the evolution of the DDI network shows that the two graph updated by the two methods have the same loss and classification results. In Figure~\ref{dW_small} (\textbf{b}), it shows the SSE of the two weight matrices, \(\mathbf{W}_{1}\) and \(\mathbf{W}_{2}\), between the reverse mode automatic differentiation and our matrix-based method is lower than \(10^{-14}\). 

Next, we show the output sensitivity with respect to the input feature matrix, computed using the result in Appendix \ref{sens_link}. In Figure \ref{Sens} (\textbf{b}), a heat map of the sensitivity of the prediction for the link between node 2 and node 7 is shown. The sensitivities of the features in node 2 and node 7 are higher than those of the features in other nodes. The features of nodes (node 1 and node 8) that do not have a connection to either node 2 or node 7, or the features of nodes (node 0) that need to traverse at least three edges to reach node 2 or node 7, have almost zero sensitivity to the link. This demonstrates the property of GCNs, where the neighborhood information of two nodes is aggregated by taking the weighted sum of the features of neighboring nodes \cite{ahmedt2022survey}. Since the GCN has only two layers, the information of nodes that are three edges away from these two nodes does not contribute to the prediction of the edge.

Additional experimental results on the DDI network are provided in Appendix \ref{Link5GCN}. Heat maps of the sensitivity of the prediction for all the links are in Figures \ref{fig:H1}, \ref{fig:H2}, and \ref{fig:H3} in the Appendix.

\begin{figure}[!t]    \includegraphics[width=\linewidth]{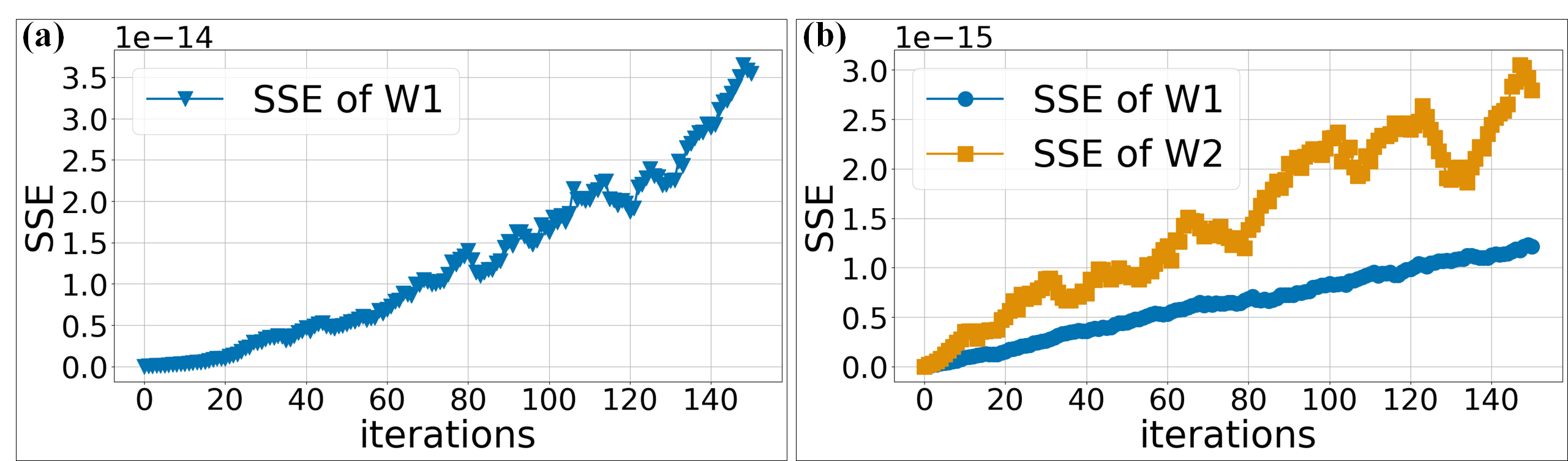}
        \caption{(\textbf{a}) The evolution of the sum of squared error between the trainable weight matrix obtained from our method and the matrix obtained using reverse mode automatic differentiation in section \ref{Node1GCN}. (\textbf{b}) The evolution of the sum of squared error between the two trainable weight matrices obtained from our method and the matrices obtained using reverse mode automatic differentiation in section \ref{Link1GCN}.}
        \label{dW_small}
\end{figure}

\begin{figure}[!t]    \includegraphics[width=\linewidth]{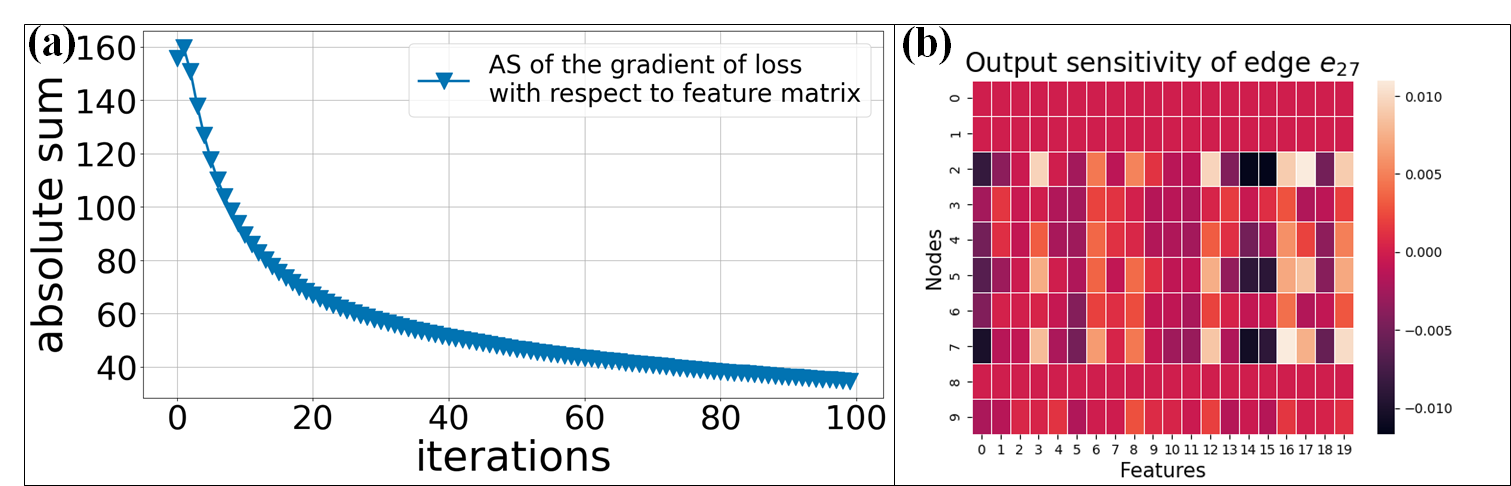}
        \caption{(\textbf{a}) The evolution of the absolute sum of the sensitivity of the loss with respect to the input feature matrix \(\mathbf{H}_{0}\) in Section \ref{Node1GCN}. (\textbf{b}) The heat map of the sensitivity of the prediction for the link between node 2 and node 7 with respect to the input feature matrix \(\mathbf{H}_{0}\) in Section \ref{Link1GCN}.}
        \label{Sens}
\end{figure}

\section{Conclusion}
In this work, we provide detailed derivations of the analytical expressions in matrix form for the derivatives of a loss function with respect to each weight matrix for a graph convolutional network considering binary node classification and link prediction. Utilizing Kronecker product, Hadamard product, and matrix calculus, we computed the gradients of the loss function in a three-layer graph convolutional network and extended the solution to accommodate graph convolutional networks with arbitrary layers and element-wise activation functions. The weight matrices obtained through our approach were compared with those obtained using reverse mode automatic differentiation in binary node classification experiments conducted on a 34 nodes Zachary’s Karate Club network \cite{zachary1977information} and in link prediction experiments using a 10-node drug-drug interaction network. The experimental results indicate that the discrepancy between the weight matrices has a median sum of squared error ranging from \(10^{-18}\) to \(10^{-14}\), affirming the accuracy of our methodology. In addition, we conduct sensitivity analysis for binary node classification and link prediction to demonstrate the application of our derivation method in XAI.

We note that our derivation already accommodates RNNs when \(\mathbf{A}\) is an identity matrix, and it accommodates CNNs since CNNs operate on 2-dimensional matrices, which are a special case of graphs. Additionally, our method incurs a significantly higher computational cost compared to reverse mode AD. As future work, we aim to find a way to improve the computational speed of back-propagation through our derived analytical solution.

\bibliographystyle{unsrt}
\bibliography{example_paper}

\appendix

\section{Basic notation and properties of Kronecker product and matrix calculus}
\label{Notations}
\subsection{Basic notation}
A column vector is denoted by lower case boldface (e.g., 
\(\mathbf{v}\), with its \(i\)th element being \(\mathbf{v}_{i}\)). Matrices is denoted by upper case boldface (e.g., \(\mathbf{A}\)). The \(i\)th row for a matrix such as \(\mathbf{A}\) is denoted \(\mathbf{A}_{i*}\) and the \(i\)th column is denoted \(\mathbf{A}_{*i}\). The \((i,j)\) element of \(\mathbf{A}\) is denoted \(a_{ij}\). An \((m \times n)\) all-ones matrix where all of its elements are equal to \(1\) is denoted as \(\mathbf{J}_{m \times n}\). An \((m \times n)\) zero matrix where all of its elements are equal to \(0\) is denoted as \(\mathbf{O}_{m \times n}\). The \((n \times n)\) identity matrix is denoted \(\mathbf{I_{n}}\). The \(k\)-dimensional column vector which has “1” in the \(j\)th row and zero elsewhere is called the unit vector and has denoted \(\underset{(k)}{\mathbf{e}_{j}}\). The \(i\)th column of a matrix \(\mathbf{A}\) can be written as:
\begin{equation}
\mathbf{A}_{*i} = \mathbf{A} \cdot \underset{(q)}{\mathbf{e}_{i}}.
\label{T5}
\end{equation}
The elementary matrix 
\begin{equation}
\mathbf{E}_{ij}^{(p\times q)} \triangleq \underset{(p)}{\mathbf{e}_{i}}\underset{(q)}{\mathbf{e}_{j}^T}
\end{equation}
has dimensions \((p\times q)\), has “1” in the \((i,j)\)th element, and has zero elsewhere.


The Kronecker product of a matrix \(\mathbf{A}\) with dimensions \((p\times q)\) and a matrix \(\mathbf{B}\) with dimensions \((m\times n)\) is represented as \(\mathbf{A}\otimes \mathbf{B}\). The resulting matrix is of size \(pm\times qn\) and is defined by
\begin{equation}
\mathbf{A}\otimes \mathbf{B} \triangleq 
\begin{bmatrix}
a_{11}\mathbf{B} & a_{12}\mathbf{B} & \cdots & a_{1q}\mathbf{B}\\
a_{21}\mathbf{B} &    &    & \vdots \\
\vdots &   &    &  \\
a_{p1}\mathbf{B} & \cdots &   & a_{pq}\mathbf{B}
\end{bmatrix}.
\label{Kronecker}
\end{equation}

The Hadamard product of \(\mathbf{A}\) \((p\times q)\) and \(\mathbf{C}\) \((p\times q)\) is denoted \(\mathbf{A}\odot \mathbf{C}\) and is a \((p\times q)\) matrix defined by
\begin{equation}
\mathbf{A} \odot \mathbf{C} \triangleq 
\begin{bmatrix}
a_{11}c_{11} & a_{12}c_{12} & \cdots & a_{1q}c_{1q}\\
a_{21}c_{21} &    &    & \vdots \\
\vdots &   &    &  \\
a_{p1}c_{p1} & \cdots &   & a_{pq}c_{pq}
\end{bmatrix}.
\label{Hadamard}
\end{equation}

The permutation matrix is a square \((pq\times pq)\) matrix defined by \cite{brewer1978kronecker}
\begin{equation}
\mathbf{U}_{p\times q} \triangleq \sum_{i}^{p} \sum_{j}^{q} \mathbf{E}_{ij}^{(p\times q)}\otimes \mathbf{E}_{ji}^{(q\times p)}.
\end{equation}
Each row and column of the square matrix \(\mathbf{U}_{p\times q}\) has only one element with a value of \(1\), and all remaining elements are zero \cite{zhu2024random}.

A related matrix is a rectangular \((p^2\times q^2)\) matrix defined by \cite{brewer1978kronecker}

\begin{equation}
\mathbf{\bar{U}}_{p\times q} \triangleq \sum_{i}^{p} \sum_{j}^{q} \mathbf{E}_{ij}^{(p\times q)}\otimes \mathbf{E}_{ij}^{(p\times q)}.
\label{RelateMatrix}
\end{equation}

The matrix derivative of a matrix function \(\mathbf{F}=[f_{ij}(\mathbf{W})]_{m \times n}: \mathbb{R}^{p \times q} \rightarrow \mathbb{R}^{m \times n}\) with respect to a matrix \(\mathbf{X}\)\((p\times q)\) is of order \((pm\times qn)\) and is defined as \cite{magnus2010concept}

\begin{equation}
\frac{\partial \mathbf{F}(\mathbf{X})}{\partial \mathbf{X}} \triangleq 
\begin{bmatrix}
\frac{\partial \mathbf{F}(\mathbf{X})}{\partial x_{11}} & \cdots & \frac{\partial \mathbf{F}(\mathbf{X})}{\partial x_{1q}}\\
\vdots & & \vdots \\
\frac{\partial \mathbf{F}(\mathbf{X})}{\partial x_{p1}} & \cdots & \frac{\partial \mathbf{F}(\mathbf{X})}{\partial x_{pq}}
\end{bmatrix},
\label{DerMatrixFunction}
\end{equation}
where
\begin{equation}
\frac{\partial \mathbf{F}(\mathbf{X})}{\partial x} \triangleq 
\begin{bmatrix}
\frac{\partial f_{11}(\mathbf{X})}{\partial x} & \cdots & \frac{\partial f_{1n}(\mathbf{X})}{\partial x}\\
\vdots & & \vdots \\
\frac{\partial f_{m1}(\mathbf{X})}{\partial x} & \cdots & \frac{\partial f_{mn}(\mathbf{X})}{\partial x}
\end{bmatrix}
\label{DerMatrixFunctionScalar}
\end{equation}
for \(x \in \mathbb{R}\).

\subsection{Properties of Kronecker product and matrix calculus}
The matrix operations related to Kronecker product are listed below and are adopted from \cite{brewer1978kronecker}. The matrices in this subsection have the following dimension: \(\mathbf{A}: \mathbb{R}^{s \times t} \rightarrow \mathbb{R}^{p \times q}\), \(\mathbf{B} \in \mathbb{R}^{s \times t}\), and \(\mathbf{C}: \mathbb{R}^{s \times t} \rightarrow \mathbb{R}^{q \times r}\).
\begin{equation}
\begin{aligned}
\frac{\partial \mathbf{A}(\mathbf{B})\mathbf{C}(\mathbf{B})}{\partial \mathbf{B}} =& \frac{\partial \mathbf{A}(\mathbf{B})}{\partial \mathbf{B}}(\mathbf{I}_{t}\otimes \mathbf{C}(\mathbf{B})) + (\mathbf{I}_{s}\otimes \mathbf{A}(\mathbf{B}))\frac{\partial \mathbf{C}(\mathbf{B})}{\partial \mathbf{B}}.
\end{aligned}
\label{T2}
\end{equation}
\begin{equation}
\frac{\partial \mathbf{A}}{\partial \mathbf{A}} = \mathbf{\bar{U}}_{p\times q}.
\label{T3}
\end{equation}
\begin{equation}
\frac{\partial \mathbf{A}^T}{\partial \mathbf{A}} = \mathbf{U}_{p\times q}.
\label{T4}
\end{equation}
The equations (\ref{T2}), (\ref{T3}), and (\ref{T4}) are employed in deriving the derivative of the loss function with respect to the weight matrix of GCN.

\section{Back-propagation of Graph Convolutional Network}


A basic graph structure is defined as:
\begin{equation}
G = (V,E),
\end{equation}
where \(|V|=n\) is the number of nodes in the graph and \(|E|=n_{e}\) is the number of edges. Denoting \(v_i \in V\) as a node and \(e_{ij} = (v_i, v_j) \in E\) as an edge pointing from \(v_i\) to \(v_j\), the adjacency matrix, \(\mathbf{A} \in {\{0,1\}}^{n \times n}\), is an \(n \times n\) matrix where \(a_{ij}\) equals \(1\) if the edge \(e_{ij}\) exists, and \(a_{ij}\) equals \(0\) if \(e_{ij}\) does not belong to \(E\); in addition, a graph may possess node attributes represented by the matrix \(\mathbf{H}_{0} \in \mathbb{R}^{n \times n_0}\), where \(\mathbf{h_{0_v}} \in \mathbb{R}^{n_0}\) is the feature vector of a node \(v\) with \(n_0\) features \cite{wu2020comprehensive}.


\section{Proof of theorem \ref{thm:arbact}}
\label{ProofTheorem}
Let \(\mathbf{F}(\mathbf{W})=[f_{ij}(\mathbf{W})]_{m \times n}: \mathbb{R}^{p \times q} \rightarrow \mathbb{R}^{m \times n}\) be a \(m\times n\) multivariate matrix-valued function of a \(p\times q\) matrix \(\mathbf{W}\in\mathbb{R}^{p \times q}\), from definition \ref{def:elemact}, \(\mathbf{\Sigma}(\mathbf{F}(\mathbf{W}))\) maps a \(m\times n\) matrix \(\mathbf{F}(\mathbf{W})\) to a \(m\times n\) matrix \(\mathbf{\Sigma}(\mathbf{F}(\mathbf{W}))\). From (\ref{DerMatrixFunction}), the derivative of \(\mathbf{\Sigma}(\mathbf{F}(\mathbf{W}))\) \((m\times n)\) with respect to \(\mathbf{W}\) \((p\times q)\) can be written as:
\begin{equation}
\frac{\partial \mathbf{\Sigma}(\mathbf{F}(\mathbf{W}))}{\partial \mathbf{W}} = 
\begin{bmatrix}
\frac{\partial \mathbf{\Sigma}(\mathbf{F}(\mathbf{W}))}{\partial w_{11}} & \cdots & \frac{\partial \mathbf{\Sigma}(\mathbf{F}(\mathbf{W}))}{\partial w_{1q}}\\
\vdots & & \vdots \\
\frac{\partial \mathbf{\Sigma}(\mathbf{F}(\mathbf{W}))}{\partial w_{p1}} & \cdots & \frac{\partial \mathbf{\Sigma}(\mathbf{F}(\mathbf{W}))}{\partial w_{pq}}
\end{bmatrix}.
\label{ActivationMatrixDerivative}
\end{equation}

Using (\ref{DerMatrixFunctionScalar}), chain rule, and definition \ref{def:derivelemact}, the first element in (\ref{ActivationMatrixDerivative}) is given by
\begin{equation}
\begin{aligned}
&\frac{\partial \mathbf{\Sigma}(\mathbf{F}(\mathbf{W}))}{\partial w_{11}} \\
&= 
\begin{bmatrix}
\frac{\partial \sigma(f_{11}(\mathbf{W}))}{\partial w_{11}} & \cdots & \frac{\partial \sigma(f_{1n}(\mathbf{W}))}{\partial w_{11}}\\
\vdots & & \vdots \\
\frac{\partial \sigma(f_{m1}(\mathbf{W}))}{\partial w_{11}} & \cdots & \frac{\partial \sigma(f_{mn}(\mathbf{W}))}{\partial w_{11}}
\end{bmatrix}\\
&=
\begin{bmatrix}
\begin{smallmatrix}
\sigma'(f_{11}(\mathbf{W}))\frac{\partial f_{11}(\mathbf{W})}{\partial w_{11}} & \cdots & \sigma'(f_{1n}(\mathbf{W}))\frac{\partial f_{1n}(\mathbf{W})}{\partial w_{11}}\\
\vdots & & \vdots \\
\sigma'(f_{m1}(\mathbf{W}))\frac{\partial f_{m1}(\mathbf{W})}{\partial w_{11}} & \cdots & \sigma'(f_{mn}(\mathbf{W}))\frac{\partial f_{mn}(\mathbf{W})}{\partial w_{11}}
\end{smallmatrix}
\end{bmatrix}\\
&= 
\mathbf{\Sigma}'(\mathbf{F}(\mathbf{W})) \odot \frac{\partial \mathbf{F}(\mathbf{W})}{\partial w_{11}} ,
\end{aligned}
\label{dsFWdw11}
\end{equation}
where \(\mathbf{\Sigma}'(\mathbf{F}(\mathbf{W}))\) is defined as
\begin{equation}
\begin{aligned}
\mathbf{\Sigma}'(\mathbf{F}(\mathbf{W})) &= 
\begin{bmatrix}
\sigma'(f_{11}(\mathbf{W})) & \cdots & \sigma'(f_{1n}(\mathbf{W}))\\
\vdots & & \vdots \\
\sigma'(f_{m1}(\mathbf{W})) & \cdots & \sigma'(f_{mn}(\mathbf{W}))
\end{bmatrix}\\
\end{aligned}.
\end{equation}
Applying the calculation of (\ref{dsFWdw11}) to all the other elements in (\ref{ActivationMatrixDerivative}), we can get
\begin{equation}
\begin{aligned}
&\frac{\partial \mathbf{\Sigma}(\mathbf{F}(\mathbf{W}))}{\partial \mathbf{W}} \\
&= 
\begin{bmatrix}
\mathbf{\Sigma}'(\mathbf{F}(\mathbf{W})) \odot \frac{\partial \mathbf{F}(\mathbf{W})}{\partial w_{11}} & \cdots & \mathbf{\Sigma}'(\mathbf{F}(\mathbf{W})) \odot \frac{\partial \mathbf{F}(\mathbf{W})}{\partial w_{1q}}\\
\vdots & & \vdots \\
\mathbf{\Sigma}'(\mathbf{F}(\mathbf{W})) \odot \frac{\partial \mathbf{F}(\mathbf{W})}{\partial w_{p1}} & \cdots & \mathbf{\Sigma}'(\mathbf{F}(\mathbf{W})) \odot \frac{\partial \mathbf{F}(\mathbf{W})}{\partial w_{pq}}
\end{bmatrix}\\
&= 
\begin{bmatrix}
\begin{smallmatrix}
\mathbf{\Sigma}'(\mathbf{F}(\mathbf{W})) & \cdots & \mathbf{\Sigma}'(\mathbf{F}(\mathbf{W}))\\
\vdots & & \vdots \\
\mathbf{\Sigma}'(\mathbf{F}(\mathbf{W})) & \cdots & \mathbf{\Sigma}'(\mathbf{F}(\mathbf{W}))
\end{smallmatrix}
\end{bmatrix}
\odot
\begin{bmatrix}
\begin{smallmatrix}
\frac{\partial \mathbf{F}(\mathbf{W})}{\partial w_{11}} & \cdots & \frac{\partial \mathbf{F}(\mathbf{W})}{\partial w_{1q}}\\
\vdots & & \vdots \\
\frac{\partial \mathbf{F}(\mathbf{W})}{\partial w_{p1}} & \cdots & \frac{\partial \mathbf{F}(\mathbf{W})}{\partial w_{pq}}
\end{smallmatrix}
\end{bmatrix}\\
&=
(\mathbf{J}_{p \times q} \otimes \mathbf{\Sigma}'(\mathbf{F}(\mathbf{W}))) \odot \frac{\partial \mathbf{F}(\mathbf{W})}{\partial \mathbf{W}}.
\end{aligned}
\label{ActivationMatrixDerivative2}
\end{equation}

\section{Back-propagation for multi-layer GCN with arbitrary activation functions}
\label{DerLink}

We consider a \(d\)-layer GCN defined as
\begin{equation}
\mathbf{\hat{Y}} = \mathbf{\Sigma}_{d+1}(\mathbf{H}_{d}\mathbf{H}_{d}^{T}),
\label{def2GCNlink}
\end{equation}
where 
\begin{equation}
\mathbf{H}_{d} = \mathbf{\Sigma}_{d}(\mathbf{\hat{A}}\mathbf{H}_{d-1}\mathbf{W}_{d}),
\label{Hdlink}
\end{equation}
\(d\in \mathbb{Z}^+\), \(\mathbf{\hat{A}}=\mathbf{\tilde{D}}^{-\frac{1}{2}}(\mathbf{A}+\mathbf{I}_{n})\mathbf{\tilde{D}}^{-\frac{1}{2}} \in\mathbb{R}^{n \times n}\) represents the \(n \times n\) normalized adjacency matrix \cite{yun2021neo}, \(\mathbf{\tilde{D}} \in\mathbb{R}^{n \times n}\) is a degree matrix defined by \(\tilde{d}_{ii} =1+\sum_{j} a_{ij}\), \(\mathbf{A} \in {\{0,1\}}^{n \times n}\) is an \(n \times n\) adjacency matrix whose diagonal elements are all equal to zero), \(\mathbf{H}_{d} \in \mathbb{R}^{n \times n_{d}}\) is the feature matrix  for \(n\) nodes with \(n_{d}\) features, \(\mathbf{\hat{Y}} \in \mathbb{R}^{n \times n_{d}}\) is the output matrix, and \(\mathbf{W}_{1} \in \mathbb{R}^{n_{0} \times n_{1}}\), \(\mathbf{W}_{2} \in \mathbb{R}^{n_{1} \times n_{2}}\), \(\cdots\), and \(\mathbf{W}_{d} \in \mathbb{R}^{n_{d-1} \times n_{d}}\) are trainable parameter matrices, and \(\mathbf{\Sigma}_{1}\), \(\mathbf{\Sigma}_{2}\), \(\cdots\), and \(\mathbf{\Sigma}_{d+1}\) are any element-wise activation function.

The loss function of the GCN for link prediction can be defined as 
\begin{equation}
L = -\sum_{(i,j)\in E} ln(\hat{y}_{ij})-\sum_{(i,j)\in S} ln(1-\hat{y}_{ij}),
\label{Linkloss}
\end{equation}
where \(y_{ij}\) denotes the element in the \(i\)th row and \(j\)th column of the training matrix \(\mathbf{Y} \in \mathbb{R}^{n \times n_{d}}\), \(\hat{y}_{ij}\) denotes the element in the \(i\)th row and \(j\)th column of \(\mathbf{\hat{Y}}\), and \(S\) is a set of edges that contains \(n_e\) negative edges \((\bar{v}_i,\bar{v}_j)\in S\) randomly sampled from \(E^c\).

The derivative of the loss in (\ref{Linkloss}) with respect to the \(s\)th weight matrix \(\mathbf{W}_{s}\) is
\begin{equation}
\frac{\partial L}{\partial \mathbf{W}_{s}} =-\sum_{(i,j)\in E} \frac{1}{\hat{y}_{ij}}\frac{\partial \hat{y}_{ij}}{\partial \mathbf{W}_{s}}+\sum_{(i,j)\in S} \frac{1}{1-\hat{y}_{ij}}\frac{\partial \hat{y}_{ij}}{\partial \mathbf{W}_{s}}, \label{2-2-1-E1-1}
\end{equation}
where \(s\in \mathbb{Z}^+\), \(s\leq d\), 
\begin{align}
&\frac{\partial \hat{y}_{ij}}{\partial \mathbf{W}_{s}} \nonumber \\
&= \mathbf{\Sigma}_{d+1}'(\mathbf{H}_{d_{i*}}\mathbf{H}_{d_{j*}}^T)\frac{\partial \mathbf{H}_{d_{i*}}\mathbf{H}_{d_{j*}}^T}{\partial \mathbf{W}_{s}} \label{2-2-1-E2-1} \\
&=\mathbf{\Sigma}_{d+1}'(\mathbf{H}_{d_{i*}}\mathbf{H}_{d_{j*}}^T) \cdot\frac{\partial \mathbf{\Sigma}_{d}(\mathbf{\hat{A}}_{i*}\mathbf{H}_{d-1}\mathbf{W}_{d})\mathbf{\Sigma}_{d}(\mathbf{\hat{A}}_{j*}\mathbf{H}_{d-1}\mathbf{W}_{d})^T}{\partial \mathbf{W}_{s}} \label{2-2-1-E2-2} \\
&=\mathbf{\Sigma}_{d+1}'(\mathbf{H}_{d_{i*}}\mathbf{H}_{d_{j*}}^T) \cdot(\frac{\partial \mathbf{\Sigma}_{d}(\mathbf{\hat{A}}_{i*}\mathbf{H}_{d-1}\mathbf{W}_{d})}{\partial \mathbf{W}_{s}} (\mathbf{I}_{n_s}\otimes\mathbf{\Sigma}_{d}(\mathbf{\hat{A}}_{j*}\mathbf{H}_{d-1}\mathbf{W}_{d})^T) \nonumber\\
&\quad+(\mathbf{I}_{n_{s-1}}\otimes\mathbf{\Sigma}_{d}(\mathbf{\hat{A}}_{i*}\mathbf{H}_{d-1}\mathbf{W}_{d})) \frac{\partial \mathbf{\Sigma}_{d}(\mathbf{\hat{A}}_{j*}\mathbf{H}_{d-1}\mathbf{W}_{d})^T}{\partial \mathbf{W}_{s}}) \label{2-2-1-E2-3} \\
&=\mathbf{\Sigma}_{d+1}'(\mathbf{H}_{d_{i*}}\mathbf{H}_{d_{j*}}^T) \nonumber \\
&\quad\cdot(((\mathbf{J}_{n_{s-1}\times n_s}\otimes\mathbf{\Sigma}_{d}'(\mathbf{\hat{A}}_{i*}\mathbf{H}_{d-1}\mathbf{W}_{d})) \odot\frac{\partial \mathbf{\hat{A}}_{i*}\mathbf{H}_{d-1}\mathbf{W}_{d}}{\partial \mathbf{W}_{s}})\cdot(\mathbf{I}_{n_s}\otimes\mathbf{\Sigma}_{d}(\mathbf{\hat{A}}_{j*}\mathbf{H}_{d-1}\mathbf{W}_{d})^T) \nonumber\\
&\quad+(\mathbf{I}_{n_{s-1}}\otimes\mathbf{\Sigma}_{d}(\mathbf{\hat{A}}_{i*}\mathbf{H}_{d-1}\mathbf{W}_{d})) \cdot((\mathbf{J}_{n_{s-1}\times n_s}\otimes\mathbf{\Sigma}_{d}'(\mathbf{\hat{A}}_{j*}\mathbf{H}_{d-1}\mathbf{W}_{d})^T)  \odot\frac{\partial (\mathbf{\hat{A}}_{j*}\mathbf{H}_{d-1}\mathbf{W}_{d})^T}{\partial \mathbf{W}_{s}})), \label{2-2-1-E2-4}
\end{align}
(\ref{2-2-1-E2-1}) follows from chain rule and definition \ref{def:derivelemact}, (\ref{2-2-1-E2-2}) follows from (\ref{Hdlink}), (\ref{2-2-1-E2-3}) follows from (\ref{T2}), (\ref{2-2-1-E2-4}) follows from theorem (\ref{thm:arbact}), 
\begin{align}
\frac{\partial \mathbf{\hat{A}}_{i*}\mathbf{H}_{d-1}\mathbf{W}_{d}}{\partial \mathbf{W}_{s}}=\begin{cases}
&\begin{aligned}
&(\mathbf{I}_{n_{s-1}} \otimes(\mathbf{\hat{A}}_{i*}\mathbf{H}_{d-1}))\mathbf{\bar{U}}_{n_{s-1}\times n_s},
\end{aligned}\quad \text{if $s=d$}\\
&\begin{aligned}
&(\mathbf{I}_{n_{s-1}} \otimes\mathbf{\hat{A}}_{i*})\frac{\partial \mathbf{H}_{d-1}}{\partial \mathbf{W}_{s}}(\mathbf{I}_{n_s}\otimes \mathbf{W}_{d}),
\end{aligned}\quad \text{if $s<d$},
\end{cases}
\label{2-2-1-E3}
\end{align}
\begin{align}
\frac{\partial (\mathbf{\hat{A}}_{j*}\mathbf{H}_{d-1}\mathbf{W}_{d})^T}{\partial \mathbf{W}_{s}} =\begin{cases}
&\begin{aligned}
&\mathbf{U}_{n_{s-1}\times n_s}(\mathbf{I}_{n_{s}} \otimes(\mathbf{\hat{A}}_{j*}\mathbf{H}_{d-1})^T),
\end{aligned}\quad \text{if $s=d$}\\
&\begin{aligned}
&(\mathbf{I}_{n_{s-1}} \otimes\mathbf{W}_{d}^T)\frac{\partial \mathbf{H}_{d-1}^T}{\partial \mathbf{W}_{s}}(\mathbf{I}_{n_s}\otimes \mathbf{\hat{A}}_{j*}^T),
\end{aligned}\quad \text{if $s<d$},
\end{cases}
\label{2-2-1-E4}
\end{align}
(\ref{2-2-1-E3}) follows from (\ref{T2}) and (\ref{T3}), (\ref{2-2-1-E4}) follows from (\ref{T2}) and (\ref{T4}), \(\frac{\partial \mathbf{H}_{d-1}}{\partial \mathbf{W}_{s}}\) is similar to (\ref{2-1-4-E3}) but with \(\mathbf{A}=\mathbf{\hat{A}}\), 
\begin{equation}
\begin{aligned}
&\frac{\partial \mathbf{H}_{d-1}^T}{\partial \mathbf{W}_{s}} \\
&=\begin{cases}
&\begin{aligned}
&(\mathbf{J}_{n_{s-1} \times n_s}\otimes \mathbf{\Sigma}_{d-1}'(\mathbf{\hat{A}}\mathbf{H}_{d-2}\mathbf{W}_{d-1})^T) \odot (\mathbf{U}_{n_{s-1}\times n_s}(\mathbf{I}_{n_{s}} \otimes(\mathbf{\hat{A}}\mathbf{H}_{d-2})^T)),
\end{aligned} \\
&\text{if $s=d-1$}\\
&\begin{aligned}
&(\mathbf{J}_{n_{s-1} \times n_s}\otimes \mathbf{\Sigma}_{d-1}'(\mathbf{\hat{A}}\mathbf{H}_{d-2}\mathbf{W}_{d-1})^T) \odot ((\mathbf{I}_{n_{s-1}} \otimes\mathbf{W}_{d-1}^T)\frac{\partial \mathbf{H}_{d-2}^T}{\partial \mathbf{W}_{s}}\cdot (\mathbf{I}_{n_s} \otimes\mathbf{\hat{A}}^T)),
\end{aligned} \\
&\text{if $s< d-1$},
\end{cases}
\end{aligned}
\label{2-2-1-E5}
\end{equation}
and (\ref{2-2-1-E5}) follows from (\ref{T2}), (\ref{T4}), and theorem (\ref{thm:arbact}).

\section{Sensitivity analysis}
\subsection{Sensitivity of loss with respect to feature matrix}
\label{sens_node}

The expression of the sensitivity of loss with respect to the feature matrix \(\mathbf{H}_{0}\) is similar to the result in section \ref{BPMLGCNAAF_NC} except that we change the dimension of the identity matrix, the all-one matrix, and the permutation related matrix, and we change the result in (\ref{2-1-3-E2}) and (\ref{2-1-4-E3}) for the last layer. Thus, the derivative, or sensitivity, of the loss in (\ref{nodeloss-4}) with respect to the input feature matrix \(\mathbf{H}_{0} \in \mathbb{R}^{n\times n_0}\) is 
\begin{align}
\frac{\partial L}{\partial \mathbf{H}_{0}} 
&=-\sum_{i=1}^{n} \sum_{j=1}^{n_{d}} (\frac{y_{ij}-\hat{y}_{ij}}{\hat{y}_{ij}(1-\hat{y}_{ij})}\mathbf{\Sigma}_{d+1}'(h_{d_{ij}}) \cdot((\mathbf{J}_{n\times n_0}\otimes \mathbf{\Sigma}_{d}'(\mathbf{A}_{i*}\mathbf{H}_{d-1}\mathbf{W}_{d_{*j}})) \nonumber \\
&\qquad\qquad\odot \frac{\partial \mathbf{A}_{i*}\mathbf{H}_{d-1}\mathbf{W}_{d_{*j}}}{\partial \mathbf{H}_{0}})) , 
\end{align}
where \(\frac{\partial \mathbf{A}_{i*}\mathbf{H}_{d-1}\mathbf{W}_{d_{*j}}}{\partial \mathbf{H}_{0}}\) is the defined as
\begin{equation}
\begin{aligned}
&\frac{\partial \mathbf{A}_{i*}\mathbf{H}_{d-1}\mathbf{W}_{d_{*j}}}{\partial \mathbf{H}_{0}} =\begin{cases}
&\begin{aligned}
&(\mathbf{I}_{n} \otimes\mathbf{A}_{i*})\cdot\mathbf{\bar{U}}_{n\times n_0}(\mathbf{I}_{n_0}\otimes \mathbf{W}_{d_{*j}})
\end{aligned}\quad \text{if $d=1$}\\
&\begin{aligned}
&(\mathbf{I}_{n} \otimes\mathbf{A}_{i*})\cdot\frac{\partial \mathbf{H}_{d-1}}{\partial \mathbf{H}_{0}}(\mathbf{I}_{n_0}\otimes \mathbf{W}_{d_{*j}})
\end{aligned}\quad \text{if $d>1$}
\end{cases}
\end{aligned},
\end{equation}
and
\begin{equation}
\begin{aligned}
&\frac{\partial \mathbf{H}_{d-1}}{\partial \mathbf{H}_{0}} \\
&=\begin{cases}
&\begin{aligned}
&(\mathbf{J}_{n\times n_0}\otimes \mathbf{\Sigma}_{d-1}'(\mathbf{A}\mathbf{H}_{d-2}\mathbf{W}_{d-1})) \odot ((\mathbf{I}_{n} \otimes\mathbf{A})\mathbf{\bar{U}}_{n\times n_0}(\mathbf{I}_{n_0} \otimes\mathbf{W}_{d-1}))
\end{aligned}\quad \\
&\text{if $d-1=1$}\\
&\begin{aligned}
&(\mathbf{J}_{n\times n_0}\otimes \mathbf{\Sigma}_{d-1}'(\mathbf{A}\mathbf{H}_{d-2}\mathbf{W}_{d-1})) \odot ((\mathbf{I}_{n} \otimes\mathbf{A})\frac{\partial \mathbf{H}_{d-2}}{\partial \mathbf{H}_{0}}(\mathbf{I}_{n_0} \otimes\mathbf{W}_{d-1}))
\end{aligned}\quad \\
&\text{if $d-1>1$}.
\end{cases}
\end{aligned}
\label{EC-1-72}
\end{equation} 

\subsection{Sensitivity of output with respect to feature matrix}
\label{sens_link}

Similarly, the derivative, or sensitivity, of each element of the output in (\ref{def2GCNlink}) with respect to the input feature matrix \(\mathbf{H}_{0} \in \mathbb{R}^{n\times n_0}\) is

\begin{align}
&\frac{\partial \hat{y}_{ij}}{\partial \mathbf{H}_{0}} \nonumber \\
&=\mathbf{\Sigma}_{d+1}'(\mathbf{H}_{d_{i*}}\mathbf{H}_{d_{j*}}^T) \nonumber \\
&\quad\cdot(((\mathbf{J}_{n\times n_0}\otimes\mathbf{\Sigma}_{d}'(\mathbf{\hat{A}}_{i*}\mathbf{H}_{d-1}\mathbf{W}_{d})) \odot\frac{\partial \mathbf{\hat{A}}_{i*}\mathbf{H}_{d-1}\mathbf{W}_{d}}{\partial \mathbf{H}_{0}})\cdot(\mathbf{I}_{n_0}\otimes\mathbf{\Sigma}_{d}(\mathbf{\hat{A}}_{j*}\mathbf{H}_{d-1}\mathbf{W}_{d})^T) \nonumber\\
&\quad+(\mathbf{I}_{n}\otimes\mathbf{\Sigma}_{d}(\mathbf{\hat{A}}_{i*}\mathbf{H}_{d-1}\mathbf{W}_{d})) \cdot((\mathbf{J}_{n\times n_0}\otimes\mathbf{\Sigma}_{d}'(\mathbf{\hat{A}}_{j*}\mathbf{H}_{d-1}\mathbf{W}_{d})^T)  \odot\frac{\partial (\mathbf{\hat{A}}_{j*}\mathbf{H}_{d-1}\mathbf{W}_{d})^T}{\partial \mathbf{H}_{0}})), 
\end{align}
where
\begin{align}
\frac{\partial \mathbf{\hat{A}}_{i*}\mathbf{H}_{d-1}\mathbf{W}_{d}}{\partial \mathbf{H}_{0}}=\begin{cases}
&\begin{aligned}
&(\mathbf{I}_{n} \otimes\mathbf{\hat{A}}_{i*})\mathbf{\bar{U}}_{n\times n_0}(\mathbf{I}_{n_0}\otimes \mathbf{W}_{d}),
\end{aligned}\quad \text{if $d=1$}\\
&\begin{aligned}
&(\mathbf{I}_{n} \otimes\mathbf{\hat{A}}_{i*})\frac{\partial \mathbf{H}_{d-1}}{\partial \mathbf{H}_{0}}(\mathbf{I}_{n_0}\otimes \mathbf{W}_{d}),
\end{aligned}\quad \text{if $d>1$},
\end{cases}
\end{align}
\begin{align}
\frac{\partial (\mathbf{\hat{A}}_{j*}\mathbf{H}_{d-1}\mathbf{W}_{d})^T}{\partial \mathbf{H}_{0}} =\begin{cases}
&\begin{aligned}
&(\mathbf{I}_{n} \otimes\mathbf{W}_{d}^T)\mathbf{U}_{n\times n_0}(\mathbf{I}_{n_{0}} \otimes(\mathbf{\hat{A}}_{j*}\mathbf{H}_{d-1})^T),
\end{aligned}\quad \text{if $d=1$}\\
&\begin{aligned}
&(\mathbf{I}_{n} \otimes\mathbf{W}_{d}^T)\frac{\partial \mathbf{H}_{d-1}^T}{\partial \mathbf{H}_{0}}(\mathbf{I}_{n_0}\otimes \mathbf{\hat{A}}_{j*}^T),
\end{aligned}\quad \text{if $d>1$},
\end{cases}
\end{align}
\(\frac{\partial \mathbf{H}_{d-1}}{\partial \mathbf{H}_{0}}\) is similar to (\ref{EC-1-72}) but with \(\mathbf{A}=\mathbf{\hat{A}}\), and
\begin{equation}
\begin{aligned}
&\frac{\partial \mathbf{H}_{d-1}^T}{\partial \mathbf{H}_{0}} \\
&=\begin{cases}
&\begin{aligned}
&(\mathbf{J}_{n\times n_0}\otimes \mathbf{\Sigma}_{d-1}'(\mathbf{\hat{A}}\mathbf{H}_{d-2}\mathbf{W}_{d-1})^T) \odot ((\mathbf{I}_{n} \otimes\mathbf{W}_{d-1}^T)\mathbf{U}_{n\times n_0}(\mathbf{I}_{n_0} \otimes\mathbf{\hat{A}}^T)),
\end{aligned} \\
&\text{if $d-1=1$}\\
&\begin{aligned}
&(\mathbf{J}_{n\times n_0}\otimes \mathbf{\Sigma}_{d-1}'(\mathbf{\hat{A}}\mathbf{H}_{d-2}\mathbf{W}_{d-1})^T) \odot ((\mathbf{I}_{n} \otimes\mathbf{W}_{d-1}^T)\frac{\partial \mathbf{H}_{d-2}^T}{\partial \mathbf{H}_{0}}\cdot (\mathbf{I}_{n_0} \otimes\mathbf{\hat{A}}^T)),
\end{aligned} \\
&\text{if $d-1>1$}.
\end{cases}
\end{aligned}
\end{equation}

\section{Data description}

\begin{figure}[!t]
    \includegraphics[width=\linewidth]{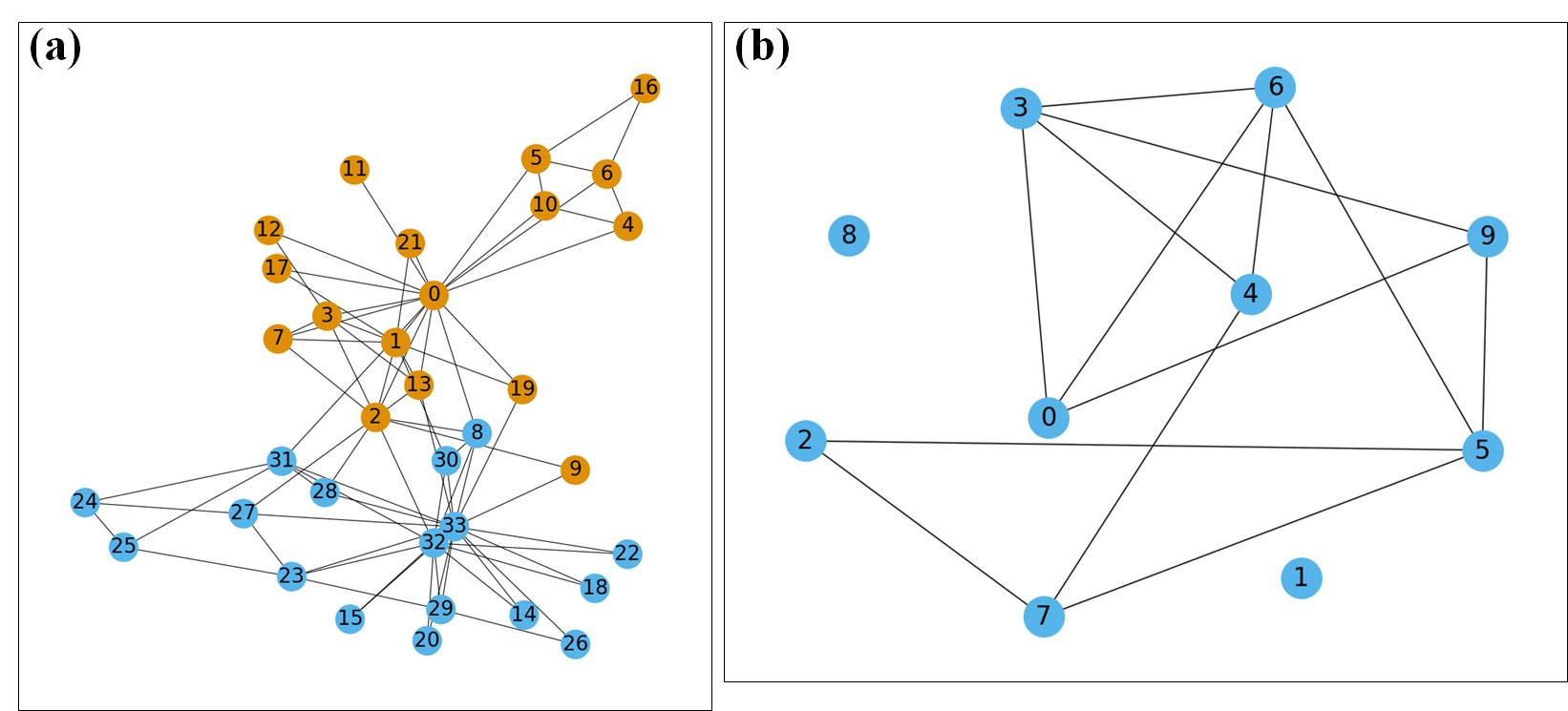}
        \caption{(\textbf{a}) Zachary’s karate club social network \cite{zachary1977information}. The node colors signify classes, with blue representing class 0 and orange representing class 1. (\textbf{b}) Drug-drug interaction network. Black links between each node represent graph edges.}
        \label{a-Karate_Drug}
\end{figure}

\subsection{Node classification}
\label{DataNode}

We tested our method for node classification on Zachary’s Karate Club \cite{zachary1977information}. Zachary’s Karate Club is a social network consisting of 34 members of a karate club, where undirected edges represent friendships, as shown in Figure~\ref{a-Karate_Drug} (\textbf{a}) \cite{moore2011active}. Every node was labeled by one of four classes obtained via modularity-based clustering in \cite{kipf2016semi} and we set class 2 to class 0 and class 3 to class 1 for binary node classification. The goal is to classify the nodes into the correct class.

\subsection{Link prediction}
\label{DataLink}

Drug-Drug Interaction (DDI) commonly arises when multiple drugs that target the same receptors are administered concurrently. Such interactions can result in diminished therapeutic efficacy by hindering drug absorption and amplifying adverse drug reactions, which may lead to unforeseen off-target side effects \cite{rodrigues2019drug}. We tested our method for link prediction using a DDI network, which was formed by extracting data from the DrugBank database \cite{wishart2018drugbank}. This network revolved around drugs of interest, identified through the analysis of differentially expressed genes from RNA-sequence data within the GSE147507 dataset, particularly in the context of a COVID-19 study \cite{blanco2020imbalanced}. Interactions between drugs and proteins were established using DrugBank's available drug-protein interactions, serving as communication pathways. The most disrupted signaling pathway, "Herpes simplex virus 1 infection," comprised 495 proteins and was selected based on the analysis of differentially expressed genes and KEGG pathway database \cite{kanehisa2021kegg}. Subsequently, 214 drugs were extracted from DrugBank for their direct interactions with the 495 proteins of interest. In total, 468 drugs (including the aforementioned 214) with 48,460 DDIs were sourced from DrugBank, capable of interacting with the 214 drugs. The DDI network was constructed using these 468 drugs. For this study, the first 100 drugs, based on their DrugBank index, were extracted to form a manageable subset for analysis.

The drug features analyzed in this study encompassed 20 molecular properties such as weight and water solubility. Furthermore, pairwise chemical similarities between drugs were computed using the Jaccard coefficient, quantifying structural similarity based on Simplified Molecular Input Line Entry System (SMILES) strings \cite{weininger1988smiles}. Principal Components Analysis was used and the first 20 (out of 488) principal components, which explained 98\% of variability in drug feature data, were used as input features to GCN.

\section{The evolution of binary node classification and link prediction}

\subsubsection{5-layer GCN with identity function and sigmoid activation function for binary node classification}
\label{Node5GCN}

The \(5\)-layer GCN is in the form of (\ref{dGCNnode}) and the loss function is defined by (\ref{nodeloss-4}), where \(d=5\) is the number of layers, \(\mathbf{H}_{0} = \mathbf{I}_{n} \in \mathbb{R}^{n \times n}\) is the feature matrix for \(n=34\) nodes with \(n=34\) features, the matrix \(\mathbf{Y} \in \mathbb{R}^{n \times n_{d}}\) represents the ground truth of nodes, \(\mathbf{\hat{Y}} \in \mathbb{R}^{n \times n_{d}}\) denotes the final layer predictions of the GCN model, \(n_d=1\) for binary node classification, \(\mathbf{W}_{1} \in \mathbb{R}^{n_{0} \times n_{1}}\), \(\mathbf{W}_{2} \in \mathbb{R}^{n_{1} \times n_{2}}\), \(\mathbf{W}_{3} \in \mathbb{R}^{n_{2} \times n_{3}}\), \(\mathbf{W}_{4} \in \mathbb{R}^{n_{3} \times n_{4}}\), and \(\mathbf{W}_{5} \in \mathbb{R}^{n_{4} \times n_{d}}\) are trainable parameter matrices, \(n_0=20\), \(n_1=2\), \(n_2=3\), \(n_3=2\), \(n_4=3\), \(\mathbf{\Sigma}_{1}\) is an element-wise ReLU function \cite{nair2010rectified}, \(\mathbf{\Sigma}_{2}\) is an element-wise sigmoid linear Unit (SiLU) function \cite{hendrycks2016gaussian}, \(\mathbf{\Sigma}_{3}\) is an element-wise exponential linear unit (ELU) function with \(\alpha=1\) \cite{clevert2015fast}, \(\mathbf{\Sigma}_{4}\) is an element-wise Leaky ReLU function \cite{maas2013rectifier}, \(\mathbf{\Sigma}_{5}\) is an element-wise identity function, and \(\mathbf{\Sigma}_{6}\) is an element-wise sigmoid function.

The GCN is trained using SGD with a learning rate of \(3\times10^{-5}\) and with 10 iterations for both methods. The GCN is trained 1060 times and the weight matrices are reinitialized every time. If the calculation of the loss function has nan value, the training at this time will be skipped and will not be counted. The line plot of the SSE of the five weight matrices, \(\mathbf{W}_{1}\), \(\mathbf{W}_{2}\), \(\mathbf{W}_{3}\), \(\mathbf{W}_{4}\), and \(\mathbf{W}_{5}\), between the reverse mode automatic differentiation and our matrix-based method is shown in Figure~\ref{node_dW} (\textbf{a}). The box plot of the \(log_{10}\) of the SSE in Figure~\ref{node_dW} (\textbf{b}) shows that the median of the five weight matrices is between \(10^{-18}\) and \(10^{-14}\). In the box plot of Figure~\ref{node_dW} (\textbf{b}), the box extends from the first quartile to the third quartile of the data, with a line at the median, while the horizontal lines extending from the box, which are called whiskers, indicating the farthest data point lying within 1.5 times the inter-quartile range from the box, and circles are those points past the end of the whiskers \cite{Hunter:2007}. 

\begin{figure*}[!h]
    \includegraphics[width=\linewidth]{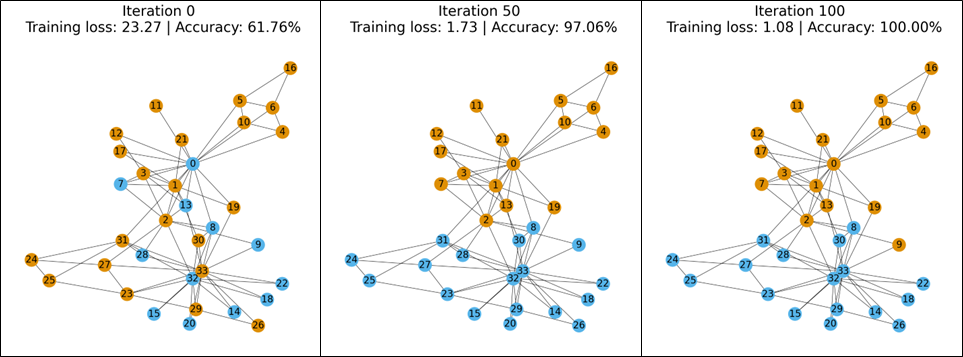}
        \caption{Evolution of karate club network node classification obtained from a 1-layer GCN model after \(100\) training iterations. The training loss, accuracy, and classification results exhibit similar trends when using either reverse-mode automatic differentiation or our matrix-based method. Nodes are represented by circle with the number to distinguish each node. Colors represent classes, where blue represents class 0 and orange represents class 1. Black links between each node represent graph edges.}
        \label{node_evolution}
\end{figure*}

\begin{figure}[!t]
    \includegraphics[width=\linewidth]{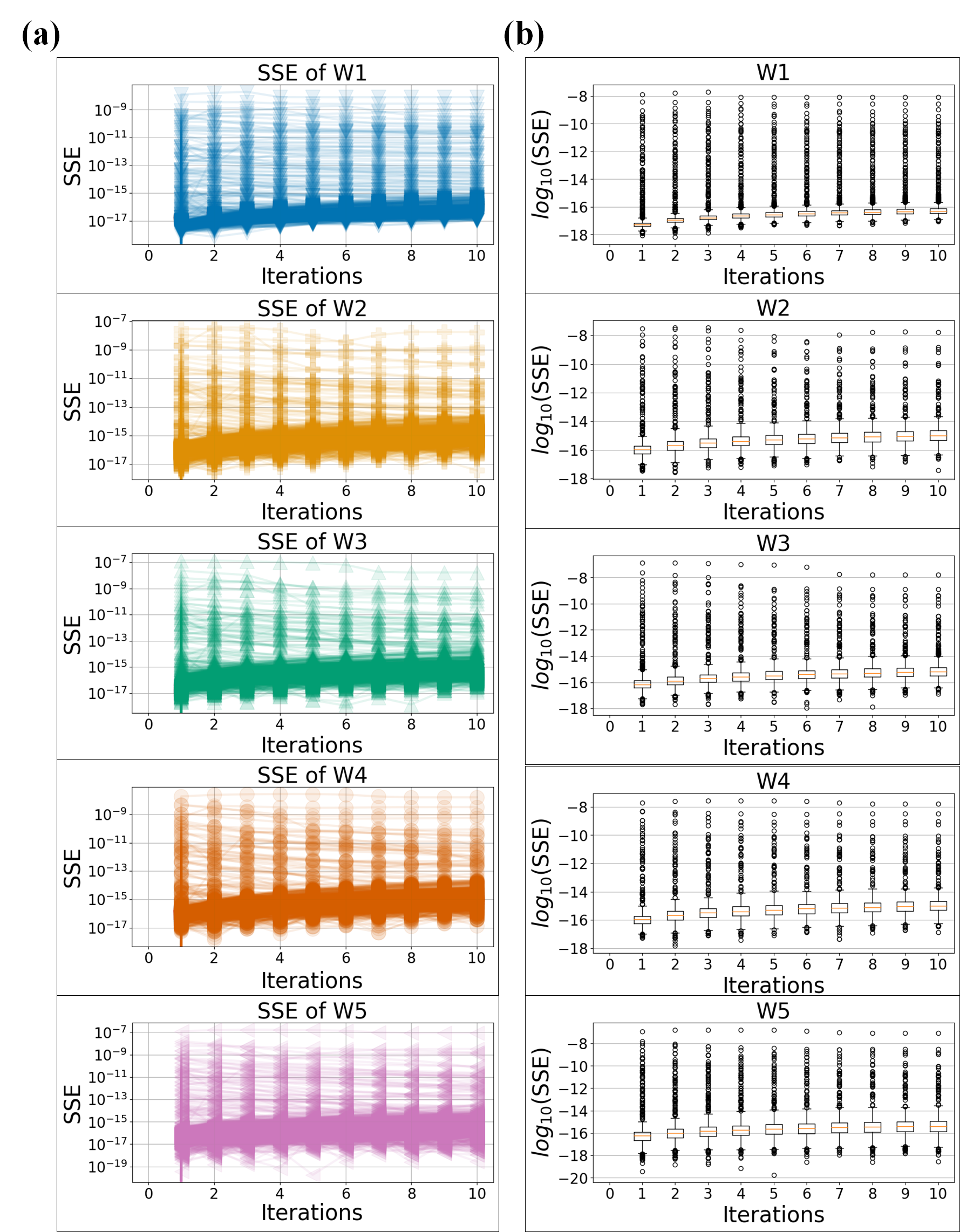}
        \caption{(\textbf{a}) The line plot of the 1060 evolution of the sum of squared error (SSE) between the trainable weight matrices obtained from our method and the matrix obtained using reverse mode automatic differentiation in section \ref{Node5GCN}.  (\textbf{b}) The box plot of the 1060 evolution of the SSE between the trainable weight matrices obtained from our method and the matrix obtained using reverse mode automatic differentiation in section \ref{Node5GCN}. 
        }
        \label{node_dW}
\end{figure}

\subsubsection{5-layer GCN for link prediction}
\label{Link5GCN}

The \(5\)-layer GCN is defined by (\ref{def2GCNlink}) and the loss function is defined by (\ref{Linkloss}), where \(d=5\) is the number of layers, the matrix \(\mathbf{\hat{Y}} \in \mathbb{R}^{n \times n}\) represents ground-truth adjacency matrix, \(\mathbf{\hat{Y}} \in \mathbb{R}^{n \times n}\) denotes the final link predictions of the GCN model, \(\mathbf{W}_{1} \in \mathbb{R}^{n_{0} \times n_{1}}\), \(\mathbf{W}_{2} \in \mathbb{R}^{n_{1} \times n_{2}}\), \(\mathbf{W}_{3} \in \mathbb{R}^{n_{2} \times n_{3}}\), \(\mathbf{W}_{4} \in \mathbb{R}^{n_{3} \times n_{4}}\), and \(\mathbf{W}_{5} \in \mathbb{R}^{n_{4} \times n_{5}}\) are trainable parameter matrices, \(n_0=20\), \(n_1=2\), \(n_2=3\), \(n_3=5\), \(n_4=3\), \(n_5=40\), \(\mathbf{\Sigma}_{1}\) is an element-wise Leaky ReLU function \cite{nair2010rectified}, \(\mathbf{\Sigma}_{2}\) is an element-wise ELU function \cite{hendrycks2016gaussian}, \(\mathbf{\Sigma}_{3}\) is an element-wise SiLU function with \(\alpha=1\) \cite{clevert2015fast}, \(\mathbf{\Sigma}_{4}\) is an element-wise ReLU function \cite{maas2013rectifier}, \(\mathbf{\Sigma}_{5}\) is an element-wise identity function, and \(\mathbf{\Sigma}_{6}\) is an element-wise sigmoid function.

The GCN is trained using SGD with a learning rate of 0.9 and with 10 iterations for both methods. In each iteration, thirteen negative edges are uniformly sampled from a set of all the unconnected edges for the GCN training using the reverse mode automatic differentiation. The sampled negative edges in each iteration are saved in a list and used in the training of the GCN using our matrix-based method. The GCN is trained 1060 times and the weight matrices are reinitialized every time. If the calculation of the loss function has nan value, the training at this time will be skipped and will not be counted. The line plot of the SSE of the five weight matrices, \(\mathbf{W}_{1}\), \(\mathbf{W}_{2}\), \(\mathbf{W}_{3}\), \(\mathbf{W}_{4}\), and \(\mathbf{W}_{5}\), between the reverse mode automatic differentiation and our matrix-based method is shown in Figure~\ref{link_dW} (\textbf{a}). The box plot of the \(log_{10}\) of the SSE in Figure~\ref{link_dW} (\textbf{b}) shows that the median of the five weight matrices is between \(10^{-18}\) and \(10^{-14}\). 

\begin{figure}[!t]
\includegraphics[width=\linewidth]{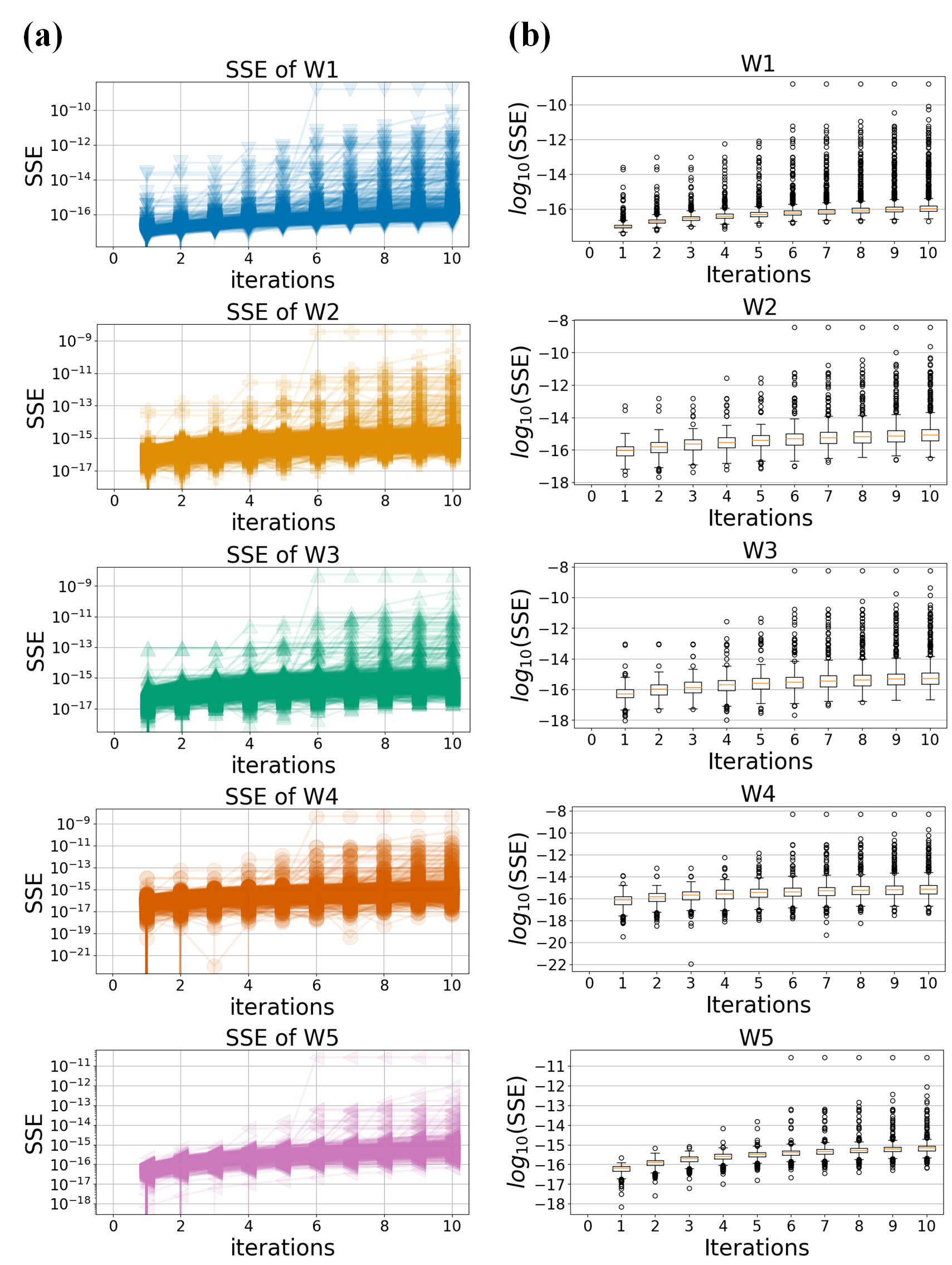}
        \caption{(\textbf{a}) The line plot of the 1060 evolution of the sum of squared error (SSE) between the trainable weight matrices obtained from our method and the matrix obtained using reverse mode automatic differentiation in section \ref{Link5GCN}.  (\textbf{b}) The box plot of the 1060 evolution of the SSE between the trainable weight matrices obtained from our method and the matrix obtained using reverse mode automatic differentiation in section \ref{Link5GCN}. 
        }
        \label{link_dW}
\end{figure}

\begin{figure*}[!t]
    \includegraphics[width=\linewidth]{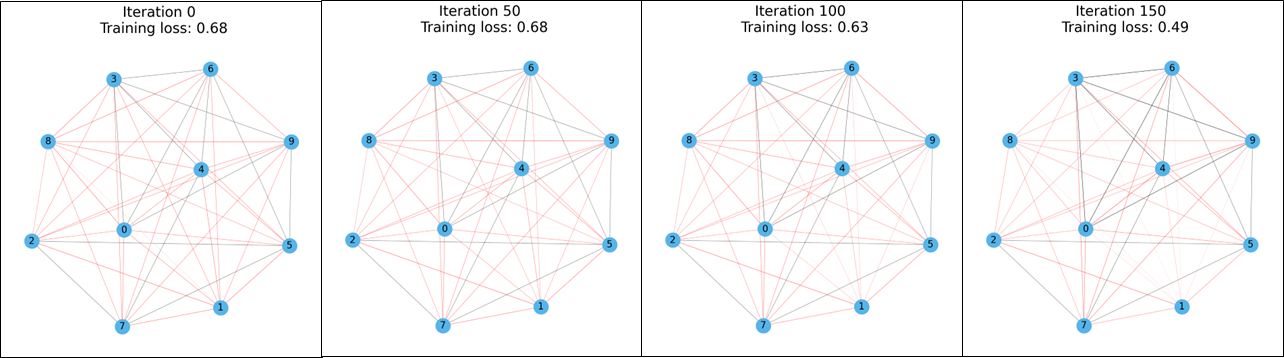}
        \caption{Evolution of DDI network node classification obtained from a 2-layer GCN model after \(150\) training iterations. The training loss and classification results exhibit similar trends when using either reverse-mode automatic differentiation or our matrix-based method. Nodes are represented by circle with the number to distinguish each node. Black links between each node represent truth graph edges. Red links between each node represent the edges that does not exist in the ground truth adjacency matrix. The transparency of the black and red link is proportional to the predicted link probability.}
        \label{link_evolution}
\end{figure*}

\begin{figure}[!t]
    \includegraphics[width=\linewidth]{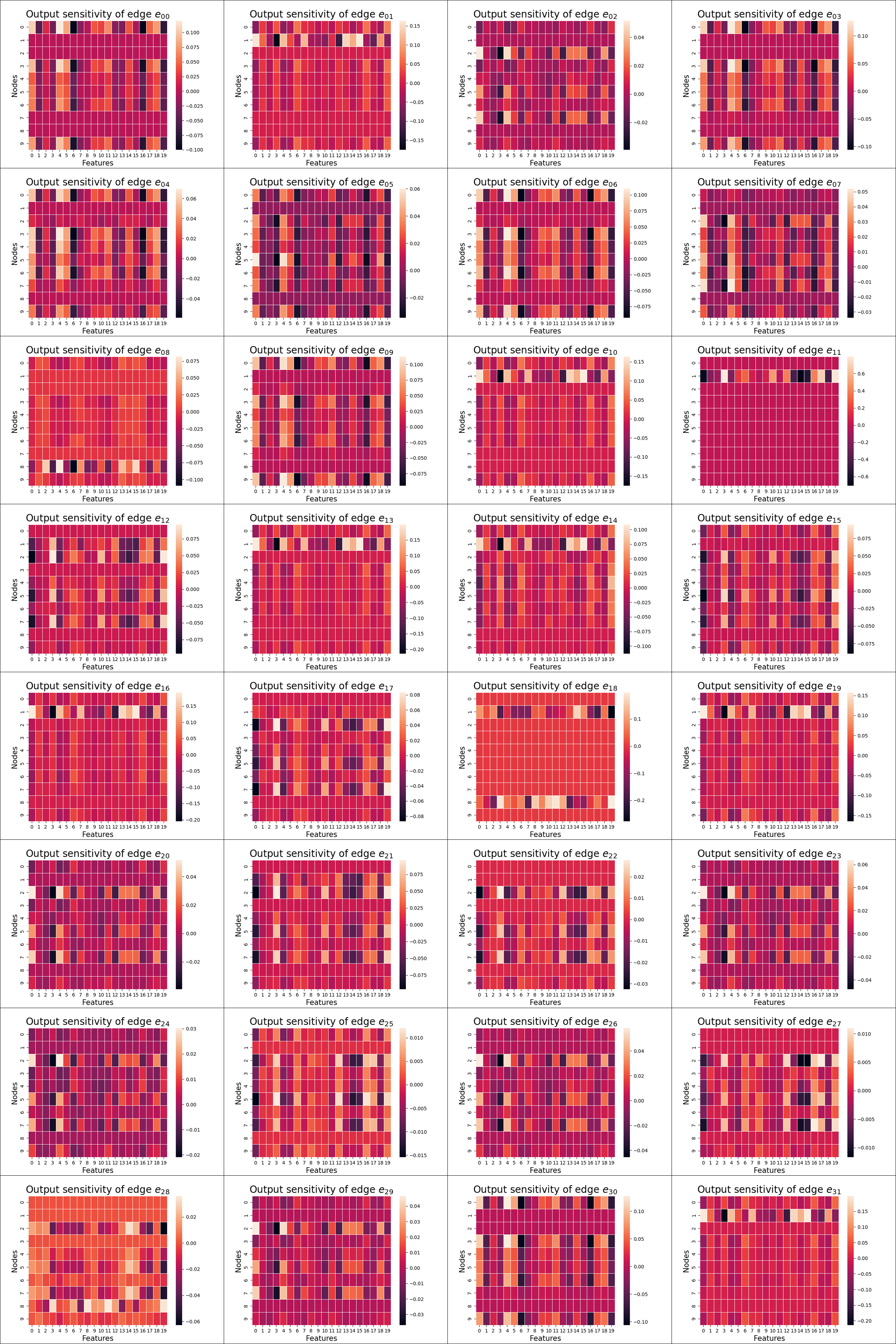}
        \caption{Heat maps of output sensitivity of the DDI network (1)}
        \label{fig:H1}
\end{figure}

\begin{figure}[!t]
    \includegraphics[width=1.01\linewidth]{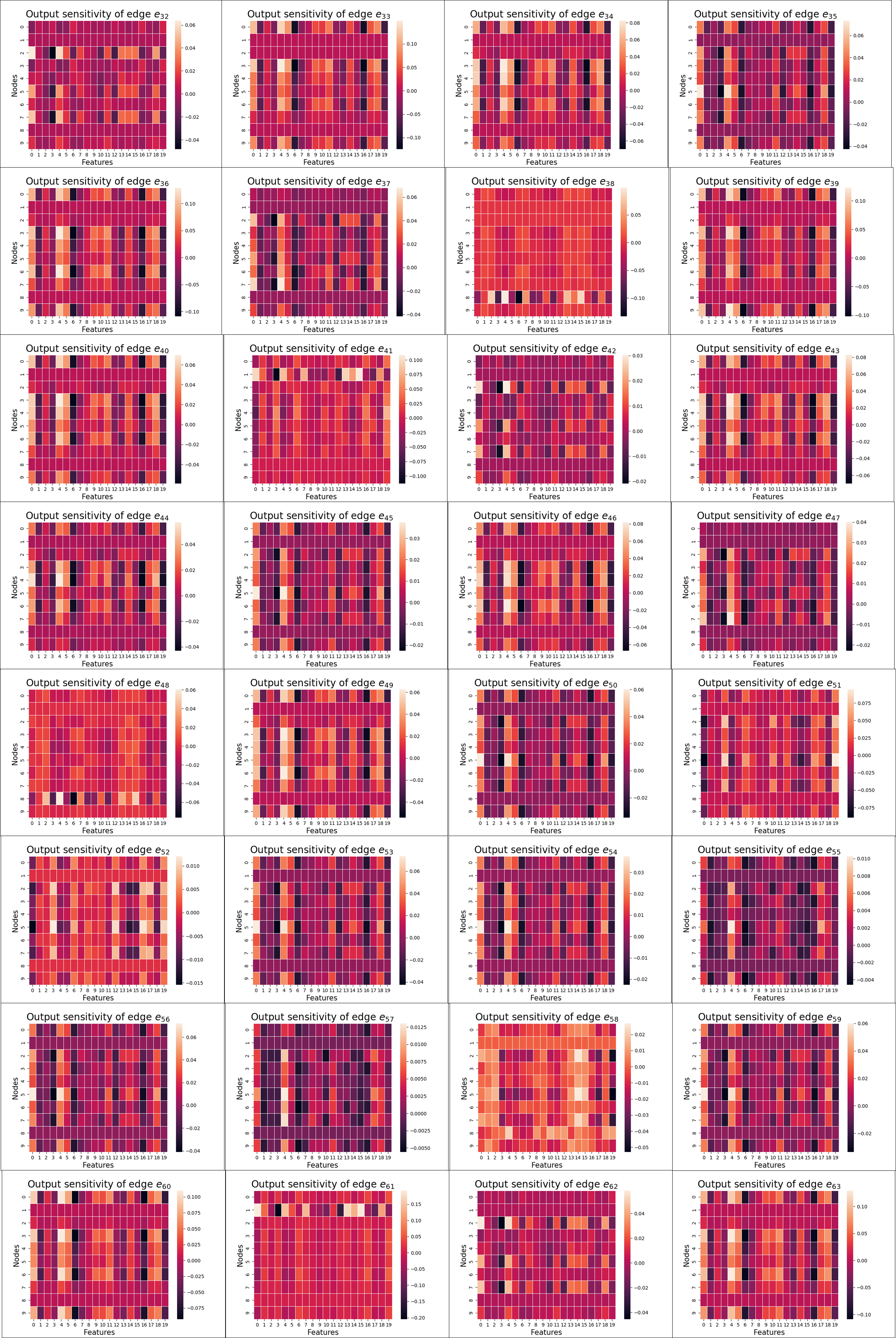}
        \caption{Heat maps of output sensitivity of the DDI network (2)}
        \label{fig:H2}
\end{figure}

\begin{figure}[!t]
    \includegraphics[width=0.91\linewidth]{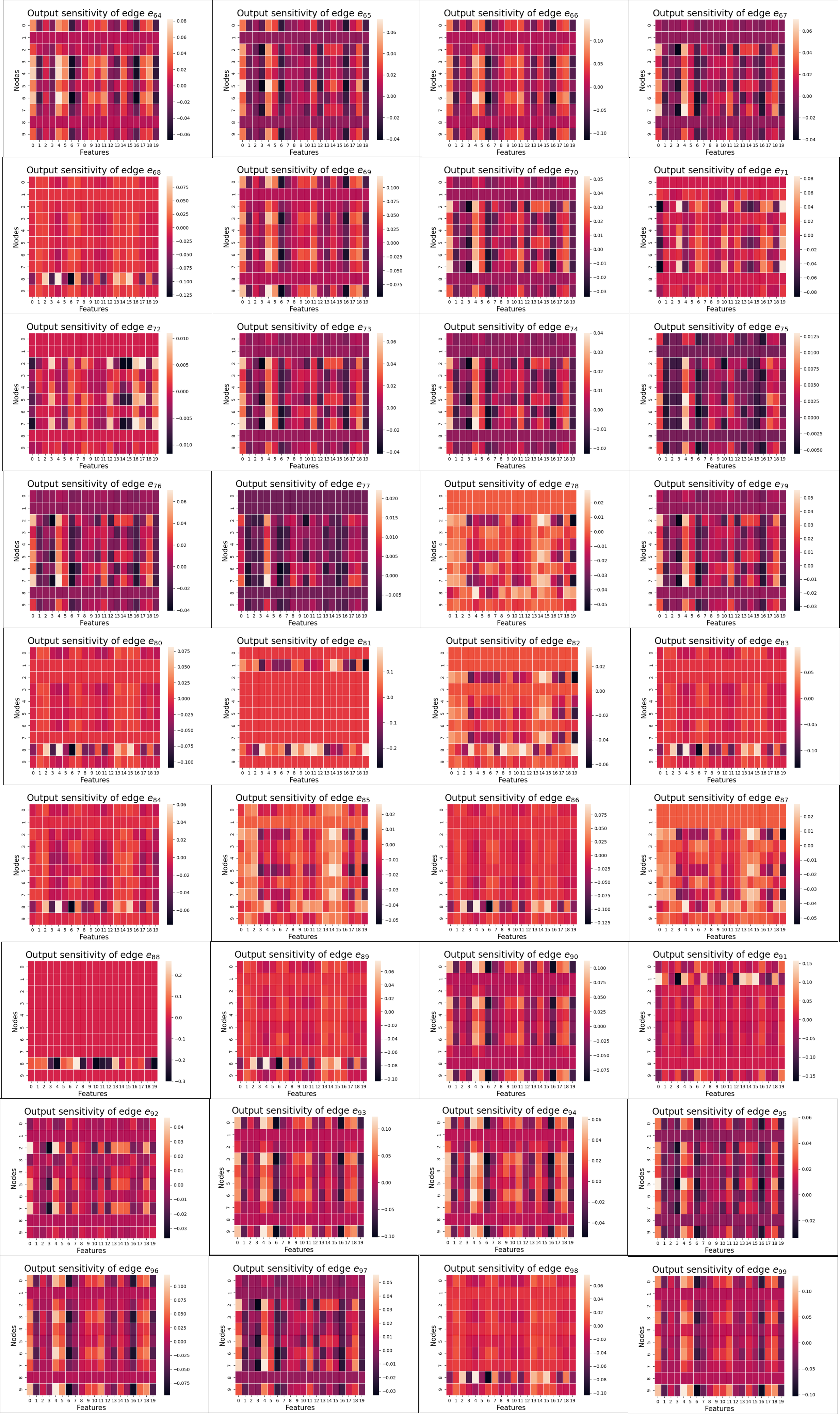}
        \caption{Heat maps of output sensitivity of the DDI network (3)}
        \label{fig:H3}
\end{figure}

\end{document}